\RequirePackage{algorithm}
\RequirePackage{algorithmic}

\documentclass[accepted]{uai2023} % for initial submission
\usepackage[american]{babel}
\usepackage{natbib} % has a nice set of citation styles and commands
\usepackage{mathtools} % amsmath with fixes and additions
\usepackage{booktabs} % commands to create good-looking tables
\usepackage{tikz} % nice language for creating drawings and diagrams

\usepackage{microtype}
\usepackage{graphicx}
\usepackage{subfigure}
\usepackage{booktabs} % for professional tables
\usepackage{svg}

\usepackage{hyperref}

\usepackage{amsmath}
\usepackage{amssymb}
\usepackage{mathtools}
\usepackage{amsthm}
\usepackage[normalem]{ulem}

\usepackage{xcolor}

\DeclareMathOperator*{\argmax}{arg\,max}
\DeclareMathOperator*{\argmin}{arg\,min}

\usepackage[capitalize,noabbrev]{cleveref}

\theoremstyle{plain}
\newtheorem{theorem}{Theorem}[section]
\newtheorem{proposition}[theorem]{Proposition}

\theoremstyle{definition}

\newtheorem{assumption}[theorem]{Assumption}
\theoremstyle{remark}

\title{Deep Gaussian Mixture Ensembles}

\author[1]{Yousef El-Laham}
\author[1]{Niccol\`o Dalmasso}
\author[1]{Elizabeth Fons}
\author[1]{Svitlana Vyetrenko}
\affil[1]{%
    J.P. Morgan AI Research, New York, USA
}
\begin{document}
\maketitle

\begin{abstract}
This work introduces a novel probabilistic deep learning technique called deep Gaussian mixture ensembles (DGMEs), which enables accurate quantification of both epistemic and aleatoric uncertainty. By assuming the data generating process follows that of a Gaussian mixture, DGMEs are capable of approximating complex probability distributions, such as heavy-tailed or multimodal distributions. Our contributions include the derivation of an expectation-maximization (EM) algorithm used for learning the model parameters, which results in an upper-bound on the log-likelihood of training data over that of standard deep ensembles.  Additionally, the proposed EM training procedure allows for learning of mixture weights, which is not commonly done in ensembles. Our experimental results demonstrate that DGMEs outperform state-of-the-art uncertainty quantifying deep learning models in handling complex predictive densities.
\end{abstract}

\section{Introduction}
\label{introduction}
Uncertainty quantification plays a key role in the development and deployment of machine learning systems, especially in applications where user safety and risk assessment are of paramount importance \citep{abdar2021review}. While deep learning (DL) has cemented its superiority in terms of raw predictive performance for a variety of applications, the principled incorporation of uncertainty quantification in DL models remains an open challenge.  Since standard DL models are unable to properly quantify predictive uncertainty, one common challenge for deep learning models is detecting out-of-distribution (OOD) inputs. It is often the case that OOD inputs lead a DL model to making erroneous predictions \citep{ovadia2019can}. Without uncertainty quantification, one cannot reason about whether an input is OOD and this can be catastrophic in applications such as machine-assisted medical decision making \citep{begoli2019need} or self-driving vehicles \citep{michelmore2018evaluating}. Moreover, uncertainty quantification can also be used as a means to select samples to label in active learning and for enabling exploration in reinforcement learning algorithms \citep{clements2019estimating, charpentier2022disentangling}.

Uncertainty in machine learning models is derived from two different sources: aleatoric uncertainty and epistemic uncertainty \citep{kendall2017uncertainties, hullermeier2019aleatoric, valdenegro2022deeper}. Aleatoric uncertainty derives from measurement process of the data, while epistemic uncertainty derives from the uncertainty in the parameters of the machine learning model. A variety of approaches have been proposed to quantify both types of uncertainty in DL models from both a Bayesian and frequentist perspective; we refer the reader to \cite{gawlikowski2021survey} for a comprehensive review. Under the Bayesian paradigm, the goal is to infer the posterior predictive density of the target variable given the input features and the training data, which encodes both types of uncertainty. Unfortunately, exact Bayesian inference algorithms (e.g., \citealt{neal2012bayesian}) cannot scale to the parameter space of modern DL architectures and one often has to resort to mini-batching \citep{chen2014stochastic} or forming a rough parametric approximation of the posterior distribution of the parameters, such as the Laplace approximation \citep{daxberger2021laplace} or stochastic variational inference \citep{graves2011practical, hoffman2013stochastic}. The drawback of the parametric approach is the inability to express more complex (e.g., heavy-tailed or multimodal) predictive distributions. As an example, approximations such as mean-field variational inference form a Gaussian predictive distribution that tends to underestimate the true uncertainty of more complex models.

In recent years, there have been developments in probabilistic DL which exploit the inherent stochasticity in learning to quantify predictive uncertainty. Examples include techniques such as probabilistic backpropagation \citep{hernandez2015probabilistic}, Monte Carlo dropout (MCD, \citealt{gal2016dropout}), Monte Carlo batch normalization \citep{teye2018bayesian}, deep ensembles (DEs, \citealt{lakshminarayanan2017de})  among others. MCD and DEs have emerged as state-of-the-art solutions for quantifying uncertainty in DL models due to their simplicity and effectiveness. MCD utilizes the inherent stochasticity of dropout (i.e., random masking of neural network weights) to form an ensemble-based approximation of the predictive distribution through multiple stochastic forward passes of the model to account for epistemic uncertainty. Aleatoric uncertainty is handled as a post-processing step under the assumption that the underlying data noise is homoscedastic.  DEs, on the other hand, independently train a small ensemble of dual-output neural networks, where the outputs characterize the mean and variance of the predictive distribution. Each network in the ensemble is independently trained to maximize the likelihood of the data under the  heteroscedastic Gaussian assumption. At test time, the networks are linearly combined into a single Gaussian approximation of the predictive distribution. Unfortunately, both MCD and DEs are not adequate solutions for modeling more complex data distributions (e.g., heavy-tailed or multimodal distributions).

{\bf Contributions.} Our contribution are as follows:

\begin{enumerate}
    \item We propose a novel probabilistic DL technique called {\it deep Gaussian mixture ensembles} (DGMEs) for jointly quantifying epistemic and aleatoric uncertainty. DGMEs train a weighted deep ensemble using the expectation maximization algorithm;
    % \item We show DGMEs optimizes the joint data likelihood directly, unlike deep ensembling that targets a lower bound of the data likelihood. As a consequence, DGMEs achieving superior loss to deep ensembling, which is corroborated by our experimental results;
     \item We show DGMEs optimizes the joint data likelihood directly, unlike deep ensembling that targets a lower bound of the data likelihood. As a consequence, DGMEs achieve a superior loss to deep ensembling, which is corroborated by our experimental results;
    \item  We empirically show that our model is more expressive than standard probabilistic DL approaches and can capture both heavy-tailedness and multimodality.
\end{enumerate}

\section{Related Work}

\textbf{Mixture Density Networks.}
Mixture density networks (MDNs, \citealt{bishop1994mixture}) use a deep neural network to simultaneously learn the means, variances and mixture weights of a Gaussian mixture model.
MDNs have been successfully used in many machine learning applications, such as computer vision~\citep{bazzani2016recurrent}, speech synthesis~\citep{zen2014deep}, probabilistic forecasting~\citep{zhang2020improved}, astronomy~\citep{d2018photometric}, chemistry~\citep{unni2020deep} and epidemiology~\citep{davis2020use}, among others. While MDNs are closely related to DGMEs in terms of uncertainty quantification, they are not an ensemble technique per se, as the epistemic and aleatoric uncertainty cannot be disentangled in MDNs. Moreover, without the ensemble structure, MDNs cannot easily be trained in a distributed setting, whereas the training of DGMEs can trivially be parallelized.

\textbf{Monte Carlo Dropout.}
MCD exploits the stochasticity of dropout training to quantify epistemic uncertainty in DL models. At test time, stochastic forward passes through a DL model with dropout produce ``approximate'' samples from the underlying posterior predictive distribution, which are typically summarized using first- and second-order moments (e.g, mean and variance of the samples). Aleatoric uncertainty is accounted for in a post-processing step, whereby the optimal homescedastic variance that maximizes the evidence lower-bound is obtained via cross-validation. MCD's popularity can be attributed to its simple implementation, as no changes to the standard DL training procedure are required. While vanilla MCD can yield favorable results in additive Gaussian settings, the method is less effective when dealing with more complex data generating processes (e.g., heavy-tailed or multimodal predictive densities). In this work, we incorporate MCD in our training procedure to account for epistemic uncertainty (see Section~\ref{sec:epistemic-uncertainty}).

\textbf{Deep Ensembles.}
DEs quantify both aleatoric and epistemic uncertainty by building an ensemble of independently trained models under different neural network weight initializations. Combined with adversarial training \citep{goodfellow2014explaining}, DEs achieve competitive or better performance than MCD in most settings in terms of calibration of predictive uncertainty and in terms of reasoning about OOD inputs. It has been argued in several works that DEs can be interpreted as a Bayesian approach, where the learned weights of each ensemble member correspond to a sample from the posterior distribution of the network weights \citep{wilson2020bayesian}. In recent years, new variations of deep ensembles have been proposed, such as anchored DEs \citep{pearce2020uncertainty}, deep-split ensembles \citep{sarawgi2020have}, and hybrid training approaches that combine DEs with the Laplace approximation \citep{hoffmann2021deep}. We emphasize that a key distinction between DEs and DGMEs is that each sample in DEs is treated as an i.i.d. sample from a Gaussian distribution, whereas DGMEs assume that data are distributed according to a Gaussian mixture. This gives DGMEs the key advantage of being able to learn more complex data generating processes.

\textbf{Neural Expectation Maximization.}
Neural EM is a differentiable clustering technique that combines the principles of the EM algorithm with neural networks for representation learning, particularly in the field of computer vision for perceptual grouping tasks \citep{NIPS2017_NEM}. The goal of neural EM is to group the individual entities (i.e., pixels) of a given input (i.e., image) that belong to the same object. To do this, a finite mixture model is used to construct a latent representation of each image, where each mixture component represents a distinct object. A neural network is then used to transform the parameters of that mixture model into pixel-wise distributions over the image, allowing for reasoning about which object each pixel in the image belongs to. While neural EM combines the ideas of EM with deep learning, we emphasize that this is different from our work which focuses on the accurate quantification of predictive uncertainty in the supervised learning setting.

\section{Problem Formulation}
\label{problem_formulation}
Consider a set of training data $\mathcal{D}=\{(x_n, y_n)\}_{n=1}^N$, where $x_n\in\mathbb{R}^{d_x}$ is the feature vector and $y_n$ is the output, which can be real-valued if we are dealing with a regression task or integer-valued if we are dealing with a classification task. We would like to train a model that allows us to predict an output $y$ given its corresponding input vector $x$. 

From a probabilistic perspective, the goal is to determine the posterior predictive distribution $p(y|x, \mathcal{D})$. We assume a statistical model $p_{\theta}(y|x)\triangleq p(y|x, \theta)$ that relates each output to its corresponding feature vector through a set of parameters $\theta\in\Theta$. Then, the predictive distribution can be determined as
\begin{equation}
    \label{eq: predictive_distribution}
    p(y|x, {\cal D})=\int_{\Theta} p_{\theta}(y|x)p(\theta|\mathcal{D})d\theta. 
\end{equation}
While this integral is generally intractable, it can be approximated using a Monte Carlo average, where samples are taken from the posterior $p(\theta|{\cal D})$. Let $Y=\{y_1,\ldots,y_N\}$ and let $X=\{x_1,\ldots, x_N\}$. According to Bayes theorem, the posterior distribution $p(\theta|{\cal D})$ is
\begin{equation}
    \label{eq: bayes_theorem_post}
    p(\theta|\mathcal{D}) = \frac{p(Y|X,\theta)p(\theta)}{p(Y|X)}, 
\end{equation}
where $p(Y|X,\theta)=\prod_{n=1}^N p_{\theta}(y_n|x_n)$ is called the {\it data likelihood} under the i.i.d. assumption, $p(\theta)$ is the {\it prior distribution} of $\theta$, and $p(Y|X)=\int_{\Theta}p(Y|X,\theta)p(\theta)d\theta$ is called the {\it marginal likelihood}. The posterior can only be computed analytically when $p(\theta)$ is a conjugate prior for the likelihood function $p(Y|X, \theta)$. For deep learning models, an analytical solution to the posterior cannot be determined and one must resort to an approximation of the predictive distribution.

\paragraph{Goal.} The goal is to acquire an approximation of the posterior predictive distribution. Ideally, samples from the approximation form consistent estimators of key moments of the predictive distribution that allow one to (i) formulate predictions, (ii) identify the underlying stochastic risk associated with the prediction (e.g., aleatoric uncertainty), and (iii) reason about the model's uncertainty in the presence of the OOD data (i.e., epistemic uncertainty).

\section{Deep Gaussian Mixture Ensembles}
\label{proposed_method}
In this work, we propose DGMEs to effectively learn a mixture distribution that accurately represents the true conditional density of the labels given the features. Since Gaussian mixtures are universal approximators for smooth probability density functions \citep{bacharoglou2010approximation, goodfellow2016deep}, modeling the conditional density $p_{\theta}(y|x)$ as a Gaussian mixture allows for learning more complex distributions, such as skewed, heavy-tailed, and multimodal distributions. Under the assumption that our data follows a mixture distribution with $K$ mixture components, the conditional density of a particular example $(x, y)$ is given by: 
\begin{equation}
    \label{eq: mixture_of_gaussains_assumptions}
    p_{\theta}(y|x) = \sum_{k=1}^K \pi_k p_k(y|x, \theta_k),
\end{equation}
where $\theta_k\in\Theta_k\subseteq \mathbb{R}^{d_{\theta}}$ denotes the underlying parameters of the $k$-th mixture and $\pi_k$ denotes the weight of the $k$-th mixture and represents the probability that the example $(x, y)$ is distributed according to $p_k(y|x, \theta_k)$. Throughout the rest of the text, we refer to all unknown parameters in the mixture as $\theta=\{\pi_1, \theta_1, ..., \pi_K, \theta_K\}$. Hereafter, we consider the problem of learning the parameters of the mixture in \eqref{eq: mixture_of_gaussains_assumptions} in the context of regression. We discuss a possible extension to classification in the Supplementary Material, Section~\ref{app: ood_detection}.

To effectively model this mixture, we make the following assumptions:
\begin{assumption}
\label{assum: equally_weighted mixture}
The mixture weights $(\pi_1, \ldots, \pi_K)\in{\cal S}_K$ do not depend on the input features, where ${\cal S}_K$ denotes the $K$-dimensional probability simplex.
\end{assumption}
\begin{assumption}
\label{assum: gaussian_mixture_assumption}
The conditional density $p_k(y|x, \theta_k)$ is a Gaussian distribution whose parameters are modeled via parameterized functions (neural networks) dependent on $x$:
\begin{equation}
\label{eq: conditional_density_assumption}
p_k(y|x, \theta_k) = {\cal N}(y; \mu_{\theta_k}(x), \sigma^2_{\theta_k}(x)),
\end{equation}
where $\theta_k$ denote the parameters of functions $\mu_{\theta_k}(\cdot)$ and $\sigma^2_{\theta_k}(\cdot)$ that output the mean and variance of the $k$-th mixture, respectively. Importantly, these functions are assumed to share parameters, just as in the original work on DEs \citep{lakshminarayanan2017de}.
\end{assumption}
Under the above assumptions, learning the mixture representation of $p_{\theta}(y|x)$ is equivalent to learning the parameters $\theta$ to maximize the data likelihood of the training examples ${\cal D}=\{(x_n, y_n)\}_{n=1}^N$.

\subsection{Learning the Mixture Parameters}
We obtain the maximum likelihood (ML) estimate or maximum a posteriori (MAP) estimate of the unknown parameters $\theta$ using the EM algorithm. Let $Y=\{y_1, \ldots, y_N\}$ and $X=\{x_1, \ldots, x_N\}$. Furthermore, let $Z=\{z_1, \ldots, z_N\}$, where each $z_n\in\{1,\ldots, K\}$ is a latent variable that denotes membership assignment of the training example $(x_n, y_n)$ to a particular mixture component, where $\pi_k\triangleq P_{\theta}(z_n=k)$ is the probability that the example $(x_n, y_n)$ belongs to the $k$-th component. Assuming that the training examples are independent and identically distributed, we can write the joint likelihood as
\begin{align*}
&p_{\theta}(Y, Z|X) = \\
&\prod_{n=1}^N\prod_{k=1}^K \pi_k^{\mathbb{I}(z_n=k)}\mathcal{N}(y_n; \mu_{\theta_k}(x_n), \sigma_{\theta_k}^2(x_n))^{\mathbb{I}(z_n=k)},
\end{align*}
with corresponding log-likelihood of
\begin{align*}
&\log p_{\theta}(Y, Z|X) =  \sum_{n=1}^N\sum_{k=1}^K \mathbb{I}(z_n=k)(\log \pi_k + \ell_{\theta_k}(x_n, y_n))
\end{align*}
where $\ell_{\theta_k}(x, y)=\log\left(\mathcal{N}(y; \mu_{\theta_k}(x), \sigma_{\theta_k}^2(x))\right)$. Our goal is to solve the following optimization problem:
\begin{align}
    \label{eq: exact_data_likelihood_max}
    \theta^\star &= \argmax_{\theta} \log p_{\theta}(Y|X) \\
    &= \argmax_{\theta} \log \left(\mathbb{E}_{Z|X, Y, \theta} \left[p_{\theta}(Y, Z|X)\right]\right),
\end{align}
which we numerically solve using the EM algorithm. In the following, we describe both the expectation step (E-Step) and maximization step (M-Step) as it relates to our model. As a note, all results presented hereafter also apply to the more general problem of obtaining the MAP estimate of the parameters $\theta$.\footnote{That is, the maximizer of $\log p(Y, \theta|X)  = \log p_{\theta}(Y|X) + \log p(\theta)$, where $p(\theta)$ is the prior distribution of the mixture parameters.}

\paragraph{E-Step:} 
We update the posterior probabilities of each $z_n$ given the parameters $\theta$ and the example $(x_n, y_n)$ for each $n$, denoted by $\gamma_{k, n}\triangleq P_{\theta}(z_n=k|x_n, y_n)$. This can be done directly using Bayes' theorem:
\begin{align}
    \gamma_{k, n} &= \frac{p_{k}(y_n|x_n, \theta_k)P_{\theta}(z_n=k)}{\sum_{j=1}^K p_{j}(y_n|x_n, \theta_j)P_{\theta}(z_n=j)} \\
    &= \frac{\pi_k{\cal N}(y_n; \mu_{\theta_k}(x_n), \sigma_{\theta_k}^2(x_n))}{\sum_{j=1}^K \pi_j {\cal N}(y_n; \mu_{\theta_j}(x_n), \sigma_{\theta_j}^2(x_n))} \label{eq: posterior_updates}
\end{align}

\paragraph{M-Step:} The parameters $\theta$ are updated in the maximization step by maximizing the expected joint log-likelihood $Q(\theta, \theta')\triangleq \mathbb{E}_{Z|X, Y, \theta'} \left[\log p_{\theta}(Y, Z|X)\right]$ given the previous parameter values $\theta'$, which is equivalent to doing lower-bound maximization on the true log-likelihood \citep{minka1998expectation}. The function $Q(\theta, \theta')$ can be readily determined as:
\begin{equation} 
Q(\theta, \theta') = \sum_{n=1}^N\sum_{k=1}^K \gamma_{k, n}(\log(\pi_k) + \ell_{\theta_k}(x_n, y_n)).
\end{equation}
The optimization of the mixture weights $(\pi_1, \ldots, \pi_K)$ can be carried out analytically and done independently of optimizing the mixture parameters $\{\theta_1, \ldots, \theta_K\}$:
\begin{equation}
    (\pi_1^\star,\ldots, \pi_K^\star) = \argmax_{(\pi_1,\ldots,\pi_K)\in{\cal S}_K} Q(\theta, \theta'),
\end{equation}
where for each  $k$
\begin{equation}
    \pi_k^\star = \frac{1}{N}\sum_{n=1}^N \gamma_{k, n}.
\end{equation}
Since the mixture parameters are assumed to be parameterised by neural networks, their optimization must be carried out using stochastic optimization. It is easy to see that the optimization of each $\theta_k$ can be done independently as:
\begin{align} 
     &\hspace*{-0.1cm}\theta_k^\star = \argmax_{\theta_k\in\Theta_k} \sum_{n=1}^N \gamma_{k, n}\ell_{\theta_k}(x_n, y_n) \\
    &\hspace*{-0.1cm}= \argmin_{\theta_k\in\Theta_k} \sum_{n=1}^N \gamma_{k, n} \left(\log \sigma_{\theta_k}^2(x_n) + \frac{(y_n - \mu_{\theta_k}(x_n))^2}{\sigma_{\theta_k}^2(x_n)} \right) \label{eq: weighted_nll}
\end{align}
This optimization step can be thought as training a deep ensemble, where each sample $(x_n, y_n)$ is weighted by $\gamma_{k, n}$ in its negative log-likelihood contribution. 

\subsection{Implementation Details}

\begin{algorithm}[t]
   \caption{Deep Gaussian Mixture Ensembles (DGMEs)}
   \label{alg: dgme}
    \begin{algorithmic}[1]
   \STATE {\bf Inputs:} %\vspace{-0.25cm} 
   \begin{itemize}
       \item Training dataset ${\cal D}=\{(x_n, y_n)\}_{n=1}^N$ %\vspace{-0.25cm}
       \item Number of mixture components $K$ %\vspace{-0.25cm}
       \item Number of EM steps $J$
   \end{itemize}
   \STATE {\bf Initialize mixture parameters:} %\vspace{-0.25cm} 
   \begin{itemize}
       \item Sample $\theta_k^{(0)} \sim p(\theta)$ for all $k$. %\vspace{-0.25cm}
       \item Set $\pi_k^{(0)} = \frac{1}{K}$ for all $k$. %\vspace{-0.25cm}
   \end{itemize}
    \FOR{$j=1,\ldots, J$ } 
   \STATE {\bf E-Step:} Update posterior probabilities $\gamma_{k, n}^{(j)}$ according to \eqref{eq: posterior_updates} with mixture weights $\pi_{k}^{(j-1)}$ and mixture parameters $\theta_k^{(j-1)}$ for all $k$ and $n$.
   \STATE {\bf M-Step:} Update mixture weights $\pi_k^{(j)}$ and parameters $\theta_k^{(j)}$ for all $k$ as
   \begin{equation*}
        \pi_k^{(j)} = \frac{1}{N}\sum_{n=1}^N \gamma_{k, n}^{(j)}
   \end{equation*} 
   and
    \begin{equation*}
    \theta_k^{(j)} = \argmax_{\theta_k\in\Theta_k} \sum_{n=1}^N \gamma_{k, n}^{(j)} \ell_{\theta_k}(x_n, y_n)  
    \end{equation*} 
    \ENDFOR
    \STATE {\bf Return:} $\pi_k^\star = \pi_k^{(J)}$  and $\theta_k^\star = \theta_k^{(J)}$ for all $k$.  
    \end{algorithmic}
\end{algorithm}

Our implementation of DGMEs trained via the EM algorithm is summarized in Algorithm \ref{alg: dgme}. To initialize the ensemble, the parameters of each network in the ensemble are randomly initialized, while the mixture weights are assumed to be equal. The algorithm is run for $J$ steps or alternatively until some stopping criterion is met. The E-Step for updating the posterior probabilities is computed directly for each sample in the training set. In the M-Step, the updates for the mixture weights are also carried out analytically, but for mixture component parameters $\theta_k$
we use the stochastic optimization to numerically solve for the updates, as an analytical solution is not available. At round $j$, we initialize each network to $\theta_k^{(j-1)}$ and then run the Adam optimizer for $E$ epochs to minimize the weighted negative log-likelihood in \eqref{eq: weighted_nll}, where the weights are given by $\gamma_{k, n}^{(j)}$ for all $n$. Finally, we note that the computational complexity of each EM step is equivalent to that of DEs and the overall time complexity scales linearly with the number of EM steps.

\subsection{Quantifying Epistemic Uncertainty} \label{sec:epistemic-uncertainty}
It is important to highlight that up until this point, we have not explicitly considered epistemic uncertainty in DGMEs. This is because the operation of training DGMEs according to Algorithm \ref{alg: dgme} yields a single set of parameters of the assumed Gaussian mixture model.\footnote{This point highlights the intrinsic difference in training DEs versus training DGMEs. DGMEs do not have a ``Bayesian" interpretation, because the EM algorithm used to train them only outputs a single set of possible parameters for the DGMEs (i.e., the corresponding posterior distribution of the weights is a  Dirac measure centered at the learned parameter values).} To account for model uncertainty, we need to account for the uncertainty in the parameters of the mixture (i.e., the mixture weights and/or the weights of the ensemble neural networks).  One simple way to do this is to apply MCD to the training procedure of DGMEs --- although we emphasize other techniques can be applied to account for epistemic uncertainty (e.g., Laplace approximation or a variational approximation to the posterior parameters).

Let $a_k = [a_{k, 1},\ldots, a_{k, d_{\theta}}]^\intercal \in \{0, 1\}^{d_\theta}$ denote a random binary vector of the same size as each $\theta_k$ and let $p_d\in[0, 1]$ denote a fixed dropout probability. Also, let $\theta^\star=\{\pi_1^\star, \theta_1^\star, \ldots,\pi_K^\star, \theta_K^\star\}$ denote the parameters learned by running Algorithm \ref{alg: dgme} with dropout incorporated in the training in the M-Step. For a given mixture component $k$, samples from the approximate posterior distribution of $\theta_k$ learned via dropout can be obtained via the following procedure: 
\begin{align*}
    a_{k, i} &\sim {\rm Bernoulli}(p_d), \quad i=1,\ldots, d_{\theta}, \\
    \theta_k &= a_k \odot \theta_k^\star,
\end{align*}
where $\odot$ denotes a Hadamard (or element-wise) product. It follows that a sample from the predictive distribution can directly be obtained as follows:
\begin{align}
    k &\sim {\rm Categorical}(\pi_1,\ldots, \pi_K),  \label{eq: sample_categorical}\\
    a_{k, i} &\sim {\rm Bernoulli}(p_d), \quad i=1,\ldots, d_{\theta}, \label{eq: sample_mask}\\
    \theta_k &= a_k \odot \theta_k^\star, \label{eq: sample_param}\\
    y &\sim  p_k(y|x, \theta_k). \label{eq: sample_pred}
\end{align}
In this procedure, one first samples the mixing component $k$ via \eqref{eq: sample_categorical}. Then, one draws a sample from the approximate posterior distribution of the parameters of the $k$-th mixture via \eqref{eq: sample_mask}-\eqref{eq: sample_param}. Finally, a prediction can be sampled via \eqref{eq: sample_pred}. We refer the reader to the Supplementary Material, Section~\ref{sec:supp-mat-sampling} for details on the validity of this sampling procedure.

\subsection{Theoretical Insights} \label{sec:theory}

In this section we provide insights into the connection between DGMEs and DEs, along with general results on convergence of our training procedure using DGMEs. We refer the reader to the Supplementary Material, Section~\ref{sec:app-proofs} for details on the proofs of each propositions.%~\ref{sec:app-proofs}

Proposition~\ref{prop:max-lower-bound} shows that maximizing the data likelihood directly as in DGMEs achieves an equal or better likelihood than maximizing each ensemble member's likelihood separately as in DEs.

\begin{proposition} \label{prop:max-lower-bound}
Under the assumption that $\pi_i = 1/K$ for $i=1,..,K-1$,
%$(\pi_1,\ldots,\pi_K)=(1/K, \ldots, 1/K)$,
maximizing the Gaussian mixture data likelihood directly achieves better or equal joint likelihood than maximizing each ensemble member's likelihood separately.
\end{proposition}

\textit{Proof Sketch.} The result can be obtained by using Jensen's inequality on the joint log-likelihood of equation~\eqref{eq: exact_data_likelihood_max} along with the assumption.

Next, Proposition~\ref{prop:em-convergence} combines recent results on neural network convergence in regression by \citet{arora2019fine} and \citet{farrell2021deep} with classical EM analysis \citep{wu1983convergence} to give intuition on why DGMEs should converge towards the maximum of the data likelihood\footnote{We note Proposition~\ref{prop:em-convergence} covers a specific setup, in which mean and variance function estimation is performed separately by using a shared pre-trained feature extraction layer and that the true data generating process is identifiable with a mixture model to begin with. A more thorough investigation on both using a separate neural network from mean and variance, as well as the under- or over-specified case, is outside of the scope of this paper.}.

\begin{assumption}[Non-flatness of the weighted log-likelihood] \label{assumption-min-likelihood}
    Given a DGMEs with $K$ mixtures, in each EM round $t$ there exists an $\epsilon_{t,k}$ such that:

    \begin{equation*}
        \sum_{n=1}^N \gamma_{k,n}\left( \ell_{\theta^{*}}(x_n, y_n) -  \ell_{\theta^{(t)}}(x_n, y_n) \right) \geq \frac{\epsilon_{t,k}}{K},
    \end{equation*}
    where $\theta^{*}_k = \argmax_{\theta \in \Theta} \sum_{n=1}^N \gamma_{k,n} \ell_{\theta}(x_n, y_n)$. Let $\epsilon = \min_{t \in T, k \in K} \epsilon_{t,k}$.
\end{assumption}

\begin{assumption}[Smoothness of the true mean function] \label{assumption-true-function-mean}
Let $\mu(x): \mathcal{X} \to \mathbb{R}$ be the true mean function and let $X \subset \mathcal{X}$. Assume there exists some $\beta \in \mathbb{N}^+$ such that $\mu(x) \in \mathcal{W}^{\beta, \infty}(X)$, where $\mathcal{W}^{\beta, \infty}(X)$ is a $(\beta, \infty)$-Sobolev ball.
\end{assumption}

\begin{assumption}[Smoothness of the true variance function] \label{assumption-true-function-variance}
Let $\sigma(x): \mathcal{X} \to \mathbb{R}^{+}$ be the true variance function and let $X \subset \mathcal{X}$. Let $\mathbf{H}^{\infty}$ be the Graham matrix as defined by \citet[Equation 12]{arora2019fine}, and assume that there exists an $M \in \mathbb{R}$ such that $\sigma(x)^T \left(\mathbf{H}^\infty\right) \sigma(x)^T  \leq M$ for some $M \in \mathbb{R}$.
\end{assumption}

\begin{assumption}[Non-degenerate weights] 
\label{assumption-non-degenerate-weights}
In each EM iteration, the weights are positive and bounded away from zero, e.g., $\pi^{(t)}_i > \xi^{(t)}_i > 0$.
\end{assumption}

\begin{proposition} \label{prop:em-convergence}
Under assumptions \ref{assumption-min-likelihood}, \ref{assumption-true-function-mean}, \ref{assumption-true-function-variance} and \ref{assumption-non-degenerate-weights}, let the mean and variance in each ensemble model be estimated via a separate 2-layer deep ReLu network from a common feature extraction layer. Then the DGMEs EM algorithm convergences to a non-stationary point that maximizes the data likelihood with high probability.
\end{proposition}

\textit{Proof Sketch.} The result follows if one shows that $Q(\theta; \theta^{(j)})$ is an increasing function of the EM steps $j$ (\citealt{wu1983convergence}), for parameter values $\theta^{(j)}$ that are not stationary points of $Q(\theta; \theta^{(j)})$. In the DGMEs case, this corresponds to proving that the weighted log-likelihood in each ensemble increases at every round $j$. The result follows by combining assumptions on the non-flatness of the weighted log-likelihood (A.3), the smoothness of true mean function (A.4) and the smoothness of the true variance function (A.5) with the results obtained about convergence of deep ReLU networks by \citet{farrell2021deep} and \citet{arora2019fine} respectively.

Finally, Proposition~\ref{prop:conn-under-initialization} connects DGMEs and DEs, showing DEs is equivalent to a single-EM-step of DGMEs under specific neural network weights initialization. As shown in Proposition~\ref{prop:em-convergence}, the EM training of DGME improves the function $Q$ at each iteration $t$, i.e., $Q\left(\theta^{(t+1)}, \theta^{(t)}\right) \geq Q\left(\theta^{(t)}, \theta^{(t)}\right)$. Hence, the final joint DGME likelihood will be larger or equal to the joint likelihood achieved by DE.

\begin{proposition} \label{prop:conn-under-initialization}
If the weights of each ensemble member are initialized to 0 with fixed bias terms, a single EM step for DGMEs is equivalent to perform DEs.
\end{proposition}

\textit{Proof Sketch.} The initialization schema implies that mixture membership is equal across samples in the first expectation round of the EM. Hence, the first M-step consists in training $K$ separate networks with each log-likelihood contribution being weighted equally.

\section{Experiments}
\label{experiments}
We evaluate the empirical performance of DGMEs via three different numerical experiments. We compare our method to the MDNs \citep{bishop1994mixture}, MCD \citep{gal2016dropout}, and DEs \citep{lakshminarayanan2017de}. MCD and DEs are widely considered to be state-of-the-art solutions for quantifying predictive uncertainty in deep learning models and have repeatedly been used as baselines for developing new techniques. Additional results and figures can be found in the Supplementary Material, Section~\ref{sec:exp-supp}. A summary table qualitatively comparing DGMEs to the benchmarks can be found in the Supplementary Material, Section~\ref{sec:supp-mat-comp}.

\subsection{Toy Regression}
\label{experiments:toy}
Consider the following model:
\begin{equation}
    y_n = u_nx_n^3 + \epsilon_n,
\end{equation}
where $u_n\in\{-1, 1\}$ with $p_u\triangleq P(u_n=-1)$ and $\epsilon_n\sim p(\epsilon)$ for all $n=1,\ldots,N$. We generate $N=800$ training samples from this model for the training set, where the input values $x_n$ range from -4 to 4. For each considered setting, we use a learning rate of $\eta=0.01$, a batch size of 32, and $E=80$ epochs to resolve the stochastic optimization problem in the M-step. For each method, we utilize a dropout probability $p_d=0.1$ to account for epistemic uncertainty. Additionally, we generate data from this toy model under three different noise settings to demonstrate the flexibility and expressive power of DGMEs as compared to other baselines. Unless otherwise stated, we assume $K=5$ networks in each mixture model-based approaches (i.e., MDNs, DEs, and DGMEs). Experimental results are described below for each noise scenario. Additional experimental results and ablation studies are provided in the Supplementary Material, Section~\ref{app: toy_regression}.

\begin{figure*}[!ht]
    \centering
    \includegraphics[width=\linewidth, trim={0 0cm 0cm 0}, clip]{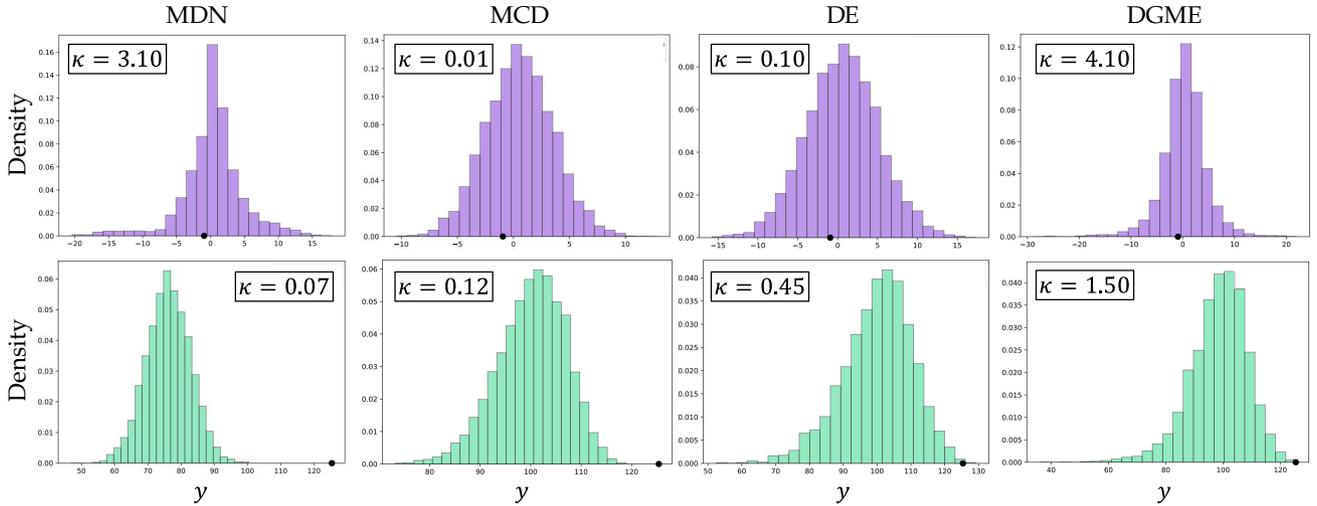}
    %\vspace*{-0.75cm}
    \caption{Histogram of samples from the predictive distributions for a single training example (top panel) and for a single test example (bottom panel) from the heavy-tailed toy regression example, shown with corresponding sample kurtosis value $\kappa$. DGMEs generally estimate heavier tailed predictions for both training and test samples, while baseline approaches samples are closer to following a Gaussian distribution.}
    \label{fig: sota_comparison_heavy_tailed}
\end{figure*}

\begin{figure*}[!ht]
    \centering
    \includegraphics[width=0.9\linewidth, trim={0 17.6cm 0 0}, clip]{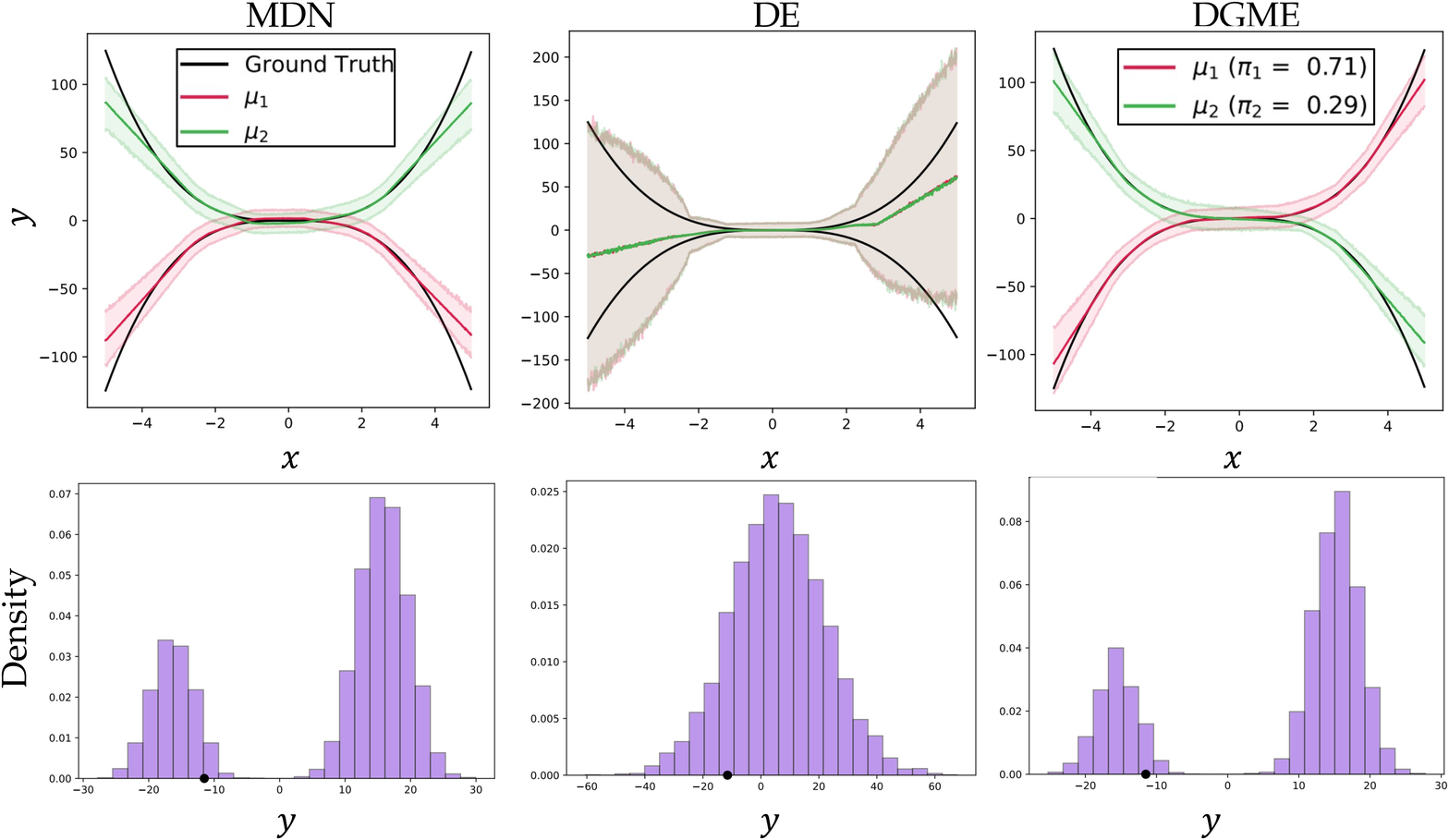}
    \caption{Predictive distribution plots for the bimodal Gaussian toy regression example. DEs cannot capture the multimodality of the noise, while MDNs and DGMEs can. Furthermore, DGMEs approximate the mixture weights of the noise accurately (ground truth: $\pi_1=0.7$ and $\pi_2=0.3$).}
    \label{fig: sota_comparison_bimodal}
\end{figure*}

\begin{table*}[t]
    \resizebox{\textwidth}{!}{
\begin{tabular}{lccccccc}
\toprule
\multicolumn{8}{c}{\textbf{\textsc{Test RMSE}}}\\
\midrule
Dataset        &          MDNs &           MCD &           DEs &   DGMEs (J=1) &   DGMEs (J=2) &   DGMEs (J=5) &  DGMEs (J=10) \\
\midrule
Boston housing &  \bf 2.79 $\pm$ 0.84 &  2.97 $\pm$ 0.85  &  3.28 $\pm$ 1.00 &  3.11 $\pm$ 0.94 &  3.00 $\pm$ 0.90 &  2.87 $\pm$ 0.86 &    2.83 $\pm$ 0.91 \\
Concrete       &  5.21 $\pm$ 0.56 &   5.23 $\pm$ 0.53 &  6.03 $\pm$ 0.58 &  5.67 $\pm$ 0.57 &  5.36 $\pm$ 0.51 &  5.20 $\pm$ 0.59 &  \bf 5.14 $\pm$ 0.58 \\
Energy         & \bf 0.71 $\pm$ 0.14 &   1.66 $\pm$ 0.19 &  2.09 $\pm$ 0.29 &  2.01 $\pm$ 0.29 &  1.79 $\pm$ 0.24 &  1.22 $\pm$ 0.25 &  1.07 $\pm$ 0.41 \\
Kin8nm         &  0.08 $\pm$ 0.00 &   0.10 $\pm$ 0.00 &  0.09 $\pm$ 0.00 &  0.08 $\pm$ 0.00 &  0.08 $\pm$ 0.00 &  \bf 0.07 $\pm$ 0.00 &  \bf 0.07 $\pm$ 0.00  \\
Power plant    &  4.12 $\pm$ 0.17 &   \bf 4.02 $\pm$ 0.18 &  4.11 $\pm$ 0.17 &  4.12 $\pm$ 0.16 &  4.10 $\pm$ 0.15 &  4.07 $\pm$ 0.15 &     4.05 $\pm$ 0.13 \\
% Protein        &          NaN &   4.36 $\pm$ 0.04 &  4.71 $\pm$ 0.06 &          NaN &          NaN &          NaN &          NaN \\
Wine           &  0.66 $\pm$ 0.04 &  \bf 0.62 $\pm$ 0.04 &  0.64 $\pm$ 0.04 &  0.63 $\pm$ 0.04 &  0.64 $\pm$ 0.04 &  0.64 $\pm$ 0.04 &      0.66 $\pm$ 0.05 \\
Yacht          &  0.96 $\pm$ 0.36 &   1.11 $\pm$ 0.38 &  1.58 $\pm$ 0.48 &  0.98 $\pm$ 0.38 &  0.85 $\pm$ 0.36 &  0.83 $\pm$ 0.40 &     \bf 0.70 $\pm$ 0.26 \\
\bottomrule
\end{tabular}
}
\caption{Average RMSE of test examples for regression experiments on real datasets. DGMEs obtain competitive or better performance in terms of RMSE on the majortity of datasets as compared to the baselines.}\label{tab:regression_experiments_RMSE}
\end{table*}

\begin{table*}[t]
    \centering
    \resizebox{\textwidth}{!}{
\begin{tabular}{lccccccc}
\toprule
\multicolumn{8}{c}{\textbf{\textsc{Test NLL}}}\\
\midrule
Dataset             &           MDNs &           MCD &            DEs &    DGMEs (J=1) &    DGMEs (J=2) &    DGMEs (J=5) &  DGMEs (J=10) \\
\midrule
Boston housing      &   2.62 $\pm$ 0.43 &  2.46 $\pm$ 0.25  &   2.41 $\pm$ 0.25 &   2.34 $\pm$ 0.19 &  \bf 2.33 $\pm$ 0.22 &   2.41 $\pm$ 0.25 &     2.46 $\pm$ 0.31 \\
Concrete            &   3.11 $\pm$ 0.26 &   3.04 $\pm$ 0.09 &   3.06 $\pm$ 0.18 &   3.04 $\pm$ 0.11 &   3.00 $\pm$ 0.12 &   2.95 $\pm$ 0.13 & \bf 2.94 $\pm$ 0.14 \\
Energy              &  \bf 1.18 $\pm$ 0.30 &   1.99 $\pm$ 0.09 &   1.38 $\pm$ 0.22 &   1.71 $\pm$ 0.19 &   1.48 $\pm$ 0.15 &   1.20 $\pm$ 0.23 &  1.20 $\pm$ 0.40 \\
Kin8nm              &  -1.18 $\pm$ 0.04 &  -0.95 $\pm$ 0.03 &  -1.20 $\pm$ 0.02 &  -1.20 $\pm$ 0.02 &  -1.23 $\pm$ 0.03 & -1.24 $\pm$ 0.02 &  \bf   -1.25 $\pm$ 0.02 \\
Power plant         &   2.81 $\pm$ 0.04 &   2.80 $\pm$ 0.05 & \bf 2.79 $\pm$ 0.04 &   2.82 $\pm$ 0.03 &   2.81 $\pm$ 0.03 &   2.81 $\pm$ 0.03 &           \bf 2.79 $\pm$ 0.03 \\
% Protein             &           NaN &   2.89 $\pm$ 0.01 &   2.83 $\pm$ 0.02 &           NaN &           NaN &           NaN &          NaN \\
Wine                &   1.01 $\pm$ 0.10 &  \bf 0.93 $\pm$ 0.06 &   0.94 $\pm$ 0.12 &   0.95 $\pm$ 0.11 &   0.96 $\pm$ 0.11 &   0.96 $\pm$ 0.12 &    1.10 $\pm$ 0.09 \\
Yacht               &   1.18 $\pm$ 0.17 &   1.55 $\pm$ 0.12 &   1.18 $\pm$ 0.21 &   1.07 $\pm$ 0.22 &   0.75 $\pm$ 0.22 &  0.60 $\pm$ 0.29 &  \bf   0.49 $\pm$ 0.29 \\
\bottomrule
\end{tabular}
}
\caption{Average NLL of test examples for regression experiments on real datasets. DGMEs obtain competitive or better performance in terms of NLL on the majority of datasets as compared to the baselines.}\label{tab:regression_experiments_NLL}
\end{table*}

\paragraph{Case 1 - Gaussian Noise:}
 We set $p_u=0$ and assume that the noise is zero-mean and Gaussian distributed with variance of $9$. This is analogous to the setup utilized in \cite{hernandez2015probabilistic}. Figure~\ref{fig: sota_comparison_gauss} (Supplementary Material, Section~\ref{sec: supp-mat-additive-gaussian}) shows the performance of DGMEs as compared to the baselines, where we observe that it outperforms MDNs and obtains comparable results to MCD and DEs. 
\paragraph{Case 2 - Heavy-tailed Noise:}
We set $p_u=0$ and assume that the noise distributed according to a zero-mean Student-t distribution with $\nu=3$ degrees of freedom with variance of $9$. Figure \ref{fig: sota_comparison_heavy_tailed} shows the histogram of samples from the predictive distribution of both a training and a test input with their corresponding sample (excess) kurtosis. We observe that on the training examples (i.e., purple histograms), only MDNs and DGMEs are able to learn the heavy-tailedness of the noise, as both MCD and DEs obtain a kurtosis close to 0. Unlike the baseline approaches, which are unable to learn the tail behavior in the test example, we observe that DGMEs is the best method at capturing the heavy-tailedness of the test examples, as it gives the largest corresponding kurtosis. 
\paragraph{Case 3 - Bimodal Gaussian Noise:}
We set $p_u=0.3$ and assume that the noise is zero-mean and Gaussian distributed with variance of $9$. For this example, we only compare the mixture-based approaches assuming $K=2$ components. Figure \ref{fig: sota_comparison_bimodal} shows the predictive density for the corresponding 99\% credible interval for each mixture in each approach, where for DGMEs we also show the learned mixture weights of each component. We observe that only MDNs and DGMEs are able to capture the bimodality of the data, with DGMEs also accurately capturing the mixture weight proportions. DEs instead overestimates the heteroscedastic variance in each network. This is due to the fact that DEs train each ensemble member independently under the assumption of Gaussian likelihood. We also show that DGMEs can robustly estimate this bimodality, even if the assumed number of mixture components is larger than 2 (see Supplementary Material, Section~\ref{sec:supp-mat-ablation-mixturen}).

\subsection{Regression on Real Datasets}
We evaluate the performance of DGMEs in regression against MDNs, MCD and DEs on a set of UCI regression benchmark datasets \citep{Dua2019UCI}; see Supplementary Material, Section~\ref{app: regression}, for further details on the datasets. 
We use the experimental setup used in \cite{hernandez2015probabilistic}, with each dataset split into 20 train-test folds. We use the same network architecture across each dataset: an MLP with a single hidden layer and ReLU activations, containing 50 hidden units. For each dataset we train for $E=40$ total epochs with a batch size of 32 and a learning rate of $\eta=0.001$. To be consistent with previous evaluations, we used $K=5$ networks in our ensemble and provide results for DGMEs for different numbers of EM steps $J\in\{1, 2, 5, 10\}$. Our results are shown in Tables \ref{tab:regression_experiments_RMSE} and \ref{tab:regression_experiments_NLL}, where we evaluate the root-mean-squared error (RMSE) and the negative log-likelihood (NLL) on the test set averaged over the different folds, respectively. In the same table, we also report the results for MDNs, MCD and DEs. Experimental results for MCD and DEs can also be found in their respective papers \citep{gal2016dropout, lakshminarayanan2017de}. Note that in this experiment, we do not apply dropout to MDNs and DGMEs and only account for the uncertainty obtained from training the models to maximize the NLL of the samples according to the Gaussian mixture assumption in \eqref{eq: mixture_of_gaussains_assumptions}.

We observe that in this experiment DGMEs are able to obtain competitive (or better) performance  with respect to the baseline methods. For certain datasets, we observe that increasing the number of EM steps greatly improves the performance (e.g., Concrete, Energy, Power Plant, and Yacht). We can see that this is not generally true for all datasets: for example, for the Boston housing dataset, increasing the number of EM steps begins to degrade the performance of the model in terms of NLL. We emphasize that performance can further be improved by incorporating dropout in the training procedure, where the dropout probability $p_d$ can be selected using cross-validation on each train-test split.

\subsection{Financial Time Series Forecasting}
For the final experiment, we focus on the task of one-step-ahead forecasting for financial time series. In particular, using historical daily price data from Yahoo finance \footnote{\url{https://finance.yahoo.com/}}, we formulate a one-step ahead forecasting problem using a long short-term memory (LSTM) network \citep{LSTMhochreiter}. The input to the network is a time series that represents the closing price of a particular stock over the past 30 trading days. The target output is the next trading day’s closing price. We assess performance of the model using two metrics: (1) the NLL of the test set, and (2) the RMSE score on the test set. We evaluate each method on three different datasets:
\begin{itemize}
    \item {\bf GOOG - stable market regime:} We use training data from the Google (GOOG) stock from the period of Jan 2019 - July 2022 and test on GOOG stock data from the period of August 2022 - January 2023. 
    \item {\bf RCL - market shock regime:} We use training data from the Royal Caribbean (RCL) stock from the period of Jan 2019 - April 2020 and test on RCL stock data from the period of May 2020 - September 2020. 
    \item {\bf GME - high volatility regime:} We use the training data from the Gamestop (GME) stock during the ``bubble" period of Nov 2020 - Jan 2022 and test on GME stock data following that period.
\end{itemize}
We ran each of the previously tested baselines and DGMEs on the three scenarios previously described. Additionally, we also test the MultiSWAG approach highlighted in \cite{NEURIPS2020MultiSWAG}, due to its effectiveness in quantifying epistemic uncertainty, which is of particular importance for the market shock regime \citep{chandra2021bayesian}. We train each model on each dataset for 5 independent runs and report the mean and standard error of both the test NLL and the test RMSE in Tables \ref{tab:forecasting_experiments_RMSE} and \ref{tab:forecasting_experiments_NLL},  where we have bolded the best performing method in each experiment according to the mean value of the metric. For details on selection of hyperparameters of each of the methods, we refer the reader to the Supplementary Material, Section~\ref{sec:supp-mat-hyperparam-tuning-finance}.

\begin{table}[t]
\centering
\resizebox{0.5\textwidth}{!}{
\begin{tabular}{@{}clllll@{}}
\toprule
\multicolumn{6}{c}{\textbf{Test RMSE}}                                                                                                          \\ \midrule
Dataset & \multicolumn{1}{c}{MDNs} & \multicolumn{1}{c}{MCD} & \multicolumn{1}{c}{DEs}  & \multicolumn{1}{c}{MultiSWAG} & \multicolumn{1}{c}{DGMEs} \\ \midrule
GOOG    & $2.74 \pm 0.06$         & $3.86 \pm 0.16$         & $2.73 \pm 0.03$          & $2.71 \pm 0.05$  & ${\bf 2.71 \pm 0.04}$               \\
RCL     & $15.01 \pm 4.71$        & $16.19 \pm 10.18$       & $14.92 \pm 1.44$              & ${\bf 11.73 \pm 0.45}$   & $14.49 \pm 2.73$       \\
GME     & $11.14 \pm 7.75$        & $2.70 \pm 0.47$         & $3.21\pm0.46$                    & ${\bf 2.00 \pm 0.06}$    & $3.19 \pm 0.33$     \\ \bottomrule
\end{tabular}}
\caption{Average RMSE of the test examples for the financial forecasting experiment.}
\label{tab:forecasting_experiments_RMSE}
\end{table}

\begin{table}[t]
\centering
\resizebox{0.5\textwidth}{!}{\begin{tabular}{@{}cccccc@{}}
\toprule
\multicolumn{6}{c}{\textbf{Test NLL}}                                                                           \\ \midrule
Dataset & MDNs               & MCD             & DEs                      & MultiSWAG    & DGMEs        \\ \midrule
GOOG    & $2.46 \pm 0.03$   & $2.98 \pm 0.01$ & $2.44 \pm 0.01$ & ${2.54 \pm 0.00}$ & ${\bf 2.43 \pm 0.02}$    \\
RCL     & $18.83 \pm 17.82$ & $6.12 \pm 3.93$ & $5.94 \pm 0.80$  & ${6.21 \pm 0.18}$ & ${\bf 5.00 \pm 0.76}$   \\
GME     & $6.01 \pm 3.85$   & $2.46 \pm 0.16$ & $2.66 \pm 0.13$   & ${\bf 2.14 \pm 0.03}$ & $2.61 \pm 0.31$   \\ \bottomrule
\end{tabular}}
\caption{Average NLL of the test examples for the financial forecasting experiment.}
\label{tab:forecasting_experiments_NLL}
\end{table}

The results indicate that for the GOOG dataset, DGMEs achieve, on average, the best NLL and RMSE score. In the case of the RCL dataset, we observe an interesting result. DGMEs attain the best performance in terms of NLL, but MultiSWAG does best in terms of RMSE. We believe DGMEs outperform in terms of NLL because the likelihood function assumed by DGMEs is a true Gaussian mixture, while the MultiSWAG approach is applying stochastic weight averaging Gaussian (SWAG) independently on multiple networks under the Gaussian likelihood assumption. This gives DGMEs the advantage in terms of learning the complex nature of the RCL dataset. On the other hand, we have found that since MultiSWAG is accounting for uncertainty using SWAG, it appears to make model training more stable (hence the smaller standard error on each of the metrics) and better accounts for epistemic uncertainty. This could possible explain why the RMSE score is lower than that of DGMEs and with smaller standard error. For the GME dataset, MultiSWAG outperforms DGMEs consistently, and with tighter standard error bars.  As a final remark, we emphasize that in our paper, we have accounted for epistemic uncertainty in DGMEs using dropout, but other methods could have been used (such as variational inference, Laplace approximation, or SWAG). Based on the results of this experiment, we highlight the possibility of incorporating SWAG in the training of DGMEs as a better way to account for epistemic uncertainty (as opposed to dropout).

\section{Conclusions}
\label{conclusions}

This paper proposes DGMEs, a novel probabilistic DL ensemble method for jointly quantifying epistemic and aleatoric uncertainty. Unlike deep ensembling, DGMEs optimizes the data likelihood directly and is able to capture complex behavior in the predictive distribution (e.g., heavy-tailedness and multimodality) by modeling the conditional distribution of the data as a Gaussian mixture. Our experiments show that DGMEs can capture more complex distributional properties than a variety of probabilistic DL baselines in regression settings and obtain competitive performance on detecting OOD samples in classification settings. As next steps, alternative mechanisms for handling the epistemic uncertainty can be considered. For example, one can instead form a variational approximation to the posterior of each mixture component, thereby  forming a Gaussian mixture approximation to the posterior parameters of the ensemble. Additionally, a more thorough analysis on the classification setting can be considered. Rather than using a mixture of categorical distributions to model the predictive density, one can use a mixture of Dirichlet distributions to account for uncertainty in the class probabilities, similar in line to the work of \cite{hobbhahn2022fast}. Finally, DGMEs can be applied to improve the efficiency of active learning algorithms and exploration strategies in reinforcement learning.

\textbf{Acknowledgments.}
This paper was prepared for informational purposes by the Artificial Intelligence Research group of JPMorgan Chase \& Co. and its affiliates (``JP Morgan''), and is not a product of the Research Department of JP Morgan. JP Morgan makes no representation and warranty whatsoever and disclaims all liability, for the completeness, accuracy or reliability of the information contained herein. This document is not intended as investment research or investment advice, or a recommendation, offer or solicitation for the purchase or sale of any security, financial instrument, financial product or service, or to be used in any way for evaluating the merits of participating in any transaction, and shall not constitute a solicitation under any jurisdiction or to any person, if such solicitation under such jurisdiction or to such person would be unlawful.

% References
%\bibliography{references}

\newpage

%Appendix
\onecolumn

\hrule height4pt
\vskip .15in

\begin{center}
{\LARGE \textbf{Deep Gaussian Mixture Ensembles}} \\[10pt]
    {\Large \textbf{SUPPLEMENTARY MATERIALS}}
\end{center}

\vskip .05in
\hrule height1pt
\vskip .25in

\appendix

\section{Theoretical Proofs}\label{sec:app-proofs}

This section includes the proofs of the propositions presented in Section~\ref{sec:epistemic-uncertainty}. We have also included the proposition statements for readability purposes.

\begin{proposition} \label{prop:max-lower-bound-app}
Under the assumption that $\pi_i = 1/K-1$ for $i=1,..,K$,
%$(\pi_1,\ldots,\pi_K)=(1/K, \ldots, 1/K)$,
maximizing the Gaussian mixture data likelihood directly achieves better or equal joint likelihood than maximizing each ensemble member's likelihood separately.
\end{proposition}

\begin{proof}
The EM algorithm minimizes the joint data log-likelihood as defined in equation~\eqref{eq: exact_data_likelihood_max}, which can be lower-bounded in the following way by using Jensen's inequality:

\begin{align*}
    \argmax_{\theta} \mathbb{E}_{X,Y} \left[ \log\left(\sum_{k=1}^K \pi_k p_k(y|x, \theta_k) \right) \right] &\geq \argmax_{\theta} \mathbb{E}_{X,Y} \left[ \sum_{k=1}^K \log\left(\pi_k \right) + \log \left( p_k(y|x, \theta_k) \right) \right] = \\
    &= \argmax_{\theta} \sum_{k=1}^K  \mathbb{E}_{X,Y} \left[\log(\pi_k) \right] + \mathbb{E}_{X,Y} \left[ \ell_{\theta_k}(x, y))\right]. \\
\end{align*}

By assumption, the first term constant (of value $-\log(K)$), hence:

\begin{align*}
    &\argmax_{\theta} \mathbb{E}_{X,Y} \left[ \log\left(\sum_{k=1}^K \pi_k p_k(y|x, \theta_k) \right) \right] \geq \argmax_{\theta} \sum_{k=1}^K \mathbb{E}_{X,Y} \left[ \ell_{\theta_k}(x, y))\right],
\end{align*}

with the lower bound corresponding to maximizing the likelihood of each ensemble member separately, as performed in DEs~\citep{lakshminarayanan2017de}.

\end{proof}

\begin{proposition} \label{prop:em-convergence-suppmat}
Under assumptions \ref{assumption-min-likelihood}, \ref{assumption-true-function-mean}, \ref{assumption-true-function-variance} and \ref{assumption-non-degenerate-weights}, let the mean and variance in each ensemble model being estimated via a separate 2-layer deep ReLu network from a common feature extraction layer. Then the DGMEs EM algorithm convergences to a non stationary point that maximizes the data likelihood with high-probability.
\end{proposition}

\begin{proof}
Using \citet[Theorem 4.1]{wu1983convergence}, to guarantee convergence of the EM algorithm it is enough to prove that at every round $t$:

\begin{equation} \label{app:diff-q-function}
\forall \theta \notin \mathcal{N}: Q(\theta^{(t+1)}; \theta^{(t)}) - Q(\theta^{(t)}; \theta^{(t)}) > 0,
\end{equation}

where $\mathcal{N}$ is the set of stationary points of the function $Q$. By writing the difference in equation~\eqref{app:diff-q-function} above we have that:

\begin{align*}
Q(\theta^{(t+1)}; \theta^{(t)}) - Q(\theta^{(t)}; \theta^{(t)}) &= \sum_{n=1}^N\sum_{k=1}^K \gamma_{k, n}(\log(\pi_k) + \ell_{\theta^{(t+1)}}(x_n, y_n)) - \sum_{n=1}^N\sum_{k=1}^K \gamma_{k, n}(\log(\pi_k) + \ell_{\theta^{(t)}}(x_n, y_n)) \\
&= \sum_{k=1}^K \left[ \sum_{n=1}^N \gamma_{k,n} \left( \ell_{\theta^{(t+1)}}(x_n, y_n)\right)  - \sum_{n=1}^N \gamma_{k,n} \left(\ell_{\theta^{(t)}}(x_n, y_n) \right) \right].
\end{align*}

By setting $\theta^{(t+1)} = \theta^{*}_k$ and using Assumption~\ref{assumption-min-likelihood}:

\begin{align*}
Q(\theta^{*}_k; \theta^{(t)}) - Q(\theta^{(t)}; \theta^{(t)}) \geq \sum_{k=1}^K \frac{\epsilon_{t,k}}{K} > \epsilon
\end{align*}

The result follows if every ensemble network can learn the maximum likelihood $\theta^*$ at every round. We will show that the above happens in high probability. Without loss of generality, set the round $t$ and the ensemble member $k$ if the mean and variance functions follow assumptions \ref{assumption-true-function-mean} and \ref{assumption-true-function-variance}. Let $\ell^* =\ell_{\theta^{*}_k}$ and $\hat{\ell}_\theta$ be the estimated likelihood. As the likelihood is Gaussian, the estimation problem is equivalent to estimating the true mean function $\mu^*(x)$ and variance function $\sigma^*(x)$. Assume the mean and variance functions are learnt independently by using a pre-trained feature extraction layer, we can break down the estimation problem into:

\begin{align*}
    \| \ell(\mu^*, \sigma^*) - \ell(\widehat{\mu}, \widehat{\sigma})\|_2 &= \| \ell(\mu^*, \sigma^*) \pm \ell(\mu^*, \hat{\sigma}) - \ell(\widehat{\mu}, \widehat{\sigma})\|_2 \\
    &\leq \underbrace{\| \ell(\mu^*, \widehat{\sigma}) - \ell(\widehat{\mu}, \widehat{\sigma})\|_2}_{(A)} + \underbrace{\| \ell(\mu^*, \sigma^*) - \ell(\mu^*, \widehat{\sigma})\|_2}_{(B)}.
\end{align*}

Provided $n > \mathcal{O}(\log(1/\delta)/ \epsilon^2)$ and using Assumption~\ref{assumption-non-degenerate-weights} to guarantee non-degenerate weights, the proposition follows since:

\begin{itemize}
    \item[(A)] For the mean function estimation, the likelihood reduces to a weighted least square loss, which satisfies the assumptions in \citet[2.1]{farrell2021deep}. Hence, one would need at least $n > \mathcal{O}(\log(1/\delta)/ \epsilon)$ samples to estimate the mean function within $\epsilon/2$ radius and with probability $1 - \delta$;
    \item[(B)] For the variance function estimation, the assumption correspond to the requirement in \citet[Section 5]{arora2019fine}; hence, one would need at least $n > \mathcal{O}(\log(1/\delta)/ \epsilon^2)$ samples to estimate the mean function within an $\epsilon/2$ radius and with probability $1 - \delta$.
\end{itemize}

\end{proof}

\begin{proposition}
If the weights of each ensemble members are initialized to 0 with fixed bias terms, a single EM step for DGMEs is equivalent to perform DEs.
\end{proposition}

\begin{proof}
    If any ensemble members $f_k$ has all weights initialized to 0, then it follows that $p_{k}(y_n|x_n, \theta_k) = a$ for some constant $\delta \in \mathbb{R}$. In addition, $\mu_{\theta_k}(x_n) = \mu$ and $\sigma_{\theta_k}^2 (x_n)$ for any $x_n$. Hence, in the expectation steps all posterior probabilities are equal to:

    \begin{align*}
        \gamma_{k, n} &= \frac{p_{k}(y_n|x_n, \theta_k)P_{\theta}(z_n=k)}{\sum_{j=1}^K p_{j}(y_n|x_n, \theta_j)P_{\theta}(z_n=j)} = \frac{\delta{\cal N}(y_n; \mu, \sigma^2)}{\sum_{j=1}^K \delta{\cal N}(y_n; \mu, \sigma^2)} = \frac{1}{K}.
    \end{align*}

    Hence, the maximization in the M-step is equal to:

    \begin{align*}
        \theta_k^\star &= \argmax_{\theta_k\in\Theta_k} \sum_{n=1}^N \gamma_{k, n}\ell_{\theta_k}(x_n, y_n) = \argmax_{\theta_k\in\Theta_k} \sum_{n=1}^N \ell_{\theta_k}(x_n, y_n),
    \end{align*}

    which corresponds to maximizing the likelihood of each ensemble member separately, as performed in DE \cite{lakshminarayanan2017de}.
    
\end{proof}

\section{Additional Experimental Results and Ablation Studies}
\label{sec:exp-supp}

\subsection{Toy Regression}
\label{app: toy_regression}

In this subsection, we provide additional experimental results and ablation studies on the toy regression dataset that provide valuable insights on the role of each of the DGME hyperparameters.

\subsubsection{Ablation: Number of EM Rounds}

To study the effect that the number of EM rounds has on training of DGMEs, Figure \ref{fig: training_improves} shows DGMEs trained with 1, 2 and 5 rounds on the toy regression task with Gaussian noise (Case 1), where the number of epochs per round is fixed to $E=80$. We can see in this figure that after $J=5$ EM rounds, the algorithm has converged to a conditional distribution that represents the ground truth quite well. 

\begin{figure}[t]
    \centering
    \includegraphics[width=0.9\linewidth]{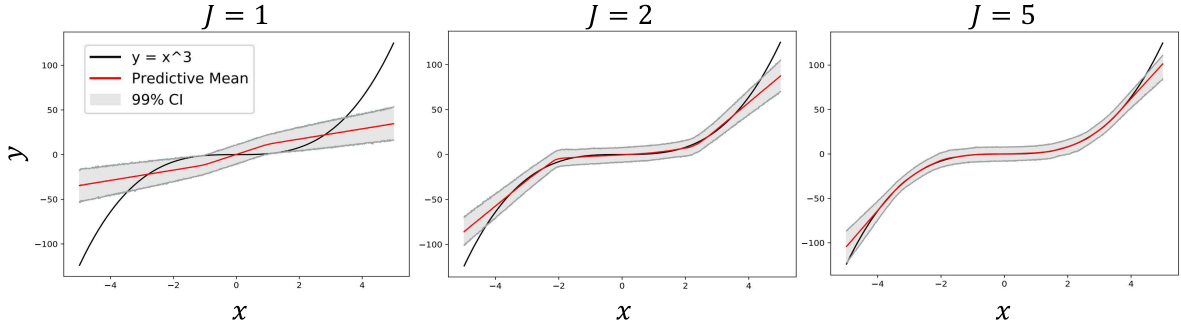}
    \caption{Results on a toy regression task with Gaussian noise for different numbers of EM rounds, as described in Section \ref{experiments:toy}. As $J$ increases, the predictive mean improves.}
    \label{fig: training_improves}
\end{figure}

Additionally, we can assess the joint impact of the number of epochs $E$ used in the M-Step per EM round and the total number of EM rounds $J$, while keeping the total computational budget constant (e.g., $E\times J =50$ total epochs). We test the following values of $E\in\{1, 2, 5, 10, 25, 50\}$ and report the average NLL over the training set and its corresponding standard error (computed over a total of 10 runs) in Table \ref{tab: em_rounds_ablation}. 

\begin{table}[t]
\centering
\resizebox{\textwidth}{!}{\begin{tabular}{@{}ccccccc@{}}
\toprule
                        & $(E=1, J = 50)$ & $(E=2, J = 25)$ & $(E=5, J = 10)$ & $(E=10, J = 5)$ & $(E=25, J = 2)$ & $(E=50, J = 1)$ \\ \midrule
Normal - Unimodal       & $2.71 \pm 0.06$ & $2.63 \pm 0.06$ & $2.58 \pm 0.03$ & $2.54 \pm 0.01$ & $2.54 \pm 0.01$ & $2.56 \pm 0.03$ \\
Heavy-Tailed - Unimodal & $2.98 \pm 0.03$ & $2.95 \pm 0.02$ & $2.88 \pm 0.02$ & $2.87 \pm 0.01$ & $2.91 \pm 0.01$ & $2.96 \pm 0.02$ \\
Normal - Bimodal        & $3.15 \pm 0.05$ & $3.09 \pm 0.07$ & $3.02 \pm 0.04$ & $3.13 \pm 0.08$ & $3.42 \pm 0.06$ & $3.53 \pm 0.04$ \\ \bottomrule
\end{tabular}}
\caption{Training NLL obtained for the toy regression dataset using DGMEs under different configurations of the number of epochs per EM round $E$ and the total number of EM rounds $J$ for a fixed computational budget $E\times J = 50$. }
\label{tab: em_rounds_ablation}
\end{table}

We can see empirically that for a fixed computational budget, there is tradeoff between the performance and the effective number of EM rounds. The tradeoff is more apparent when considering the more difficult examples (i.e., heavy-tailed unimodal noise and normal bimodal noise). If the number of epochs in the M-step is $E=1$ and we train for $J=50$ rounds, not enough information is being propagated between the E- and M-Step in each round of training, making learning inefficient. In the other extreme, if the number of epochs per M-step is $E=50$ and we train for $J=1$ rounds, even if the optimization problem in the M-step is more accurately resolved, we do not run enough EM rounds to accurately learn the underlying conditional density function. If we balance the number of epochs per rounds and the total number of EM rounds (i.e., $(E=5, J=10)$ or $(E=10, J=5)$), we get much better performance in terms of training NLL.

\subsubsection{Ablation: Dropout and Adversarial Training}
In this ablation, our goal is to understand the effect of epistemic uncertainty estimation techniques in DGMEs. As a rough analysis, Figure \ref{fig: dropout_ablation} shows the effect of training with dropout, adversarial training and their combination. Here, the dropout probability is set to $p_d=0.05$. We can see that without dropout or adversarial training, the uncertainty estimates are well-calibrated for the training data (features taking value between -4 and 4), but are underestimated for the test data (features taking absolute value between 4 and 5). By incorporating dropout and adversarial training, we can see that the uncertainty estimates become larger for the test examples.

\begin{figure}[t]
    \centering
    \includegraphics[width=\linewidth]{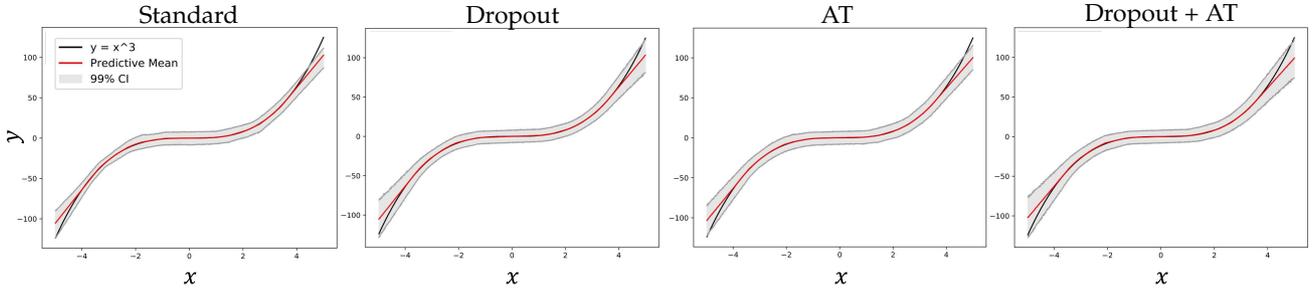}
    \caption{Results on a toy regression task with Gaussian noise. Left most plot corresponds to standard set up of DGMEs trained with $K=5$ networks. Second plot corresponds to incorporating Dropout in the training. Third plot shows the effect of using adversarial training, and final plot shows the effect of using both dropout and adversarial training.}
    \label{fig: dropout_ablation}
\end{figure}

To get a better understanding of the effect of dropout probability $p_d$ on the quantified uncertainty, we can evaluate the train and test NLL for different values of $p_d$ for each of the toy datasets. Results are shown in Table \ref{tab: dropout_ablation}. From this table, we observe that dropout creates a trade-off between performance on in-sample data and out-of-sample data in terms of NLL. Increasing the dropout probability in this case causes the average NLL to be worse for the training set, but improves it (up to a certain point) on the test set. In practice, we can choose the dropout probability to minimize the NLL on a validation set.

\begin{table}[t]
\resizebox{\textwidth}{!}{\begin{tabular}{@{}ccccccccccc@{}}
\toprule
                        & \multicolumn{2}{c}{$p_d = 0.0$}    & \multicolumn{2}{c}{$p_d = 0.05$}   & \multicolumn{2}{c}{$p_d = 0.1$}    & \multicolumn{2}{c}{$p_d = 0.15$}   & \multicolumn{2}{c}{$p_d = 0.2$}    \\ \midrule
                        & Train NLL       & Test NLL         & Train NLL       & Test NLL         & Train NLL       & Test NLL         & Train NLL       & Test NLL         & Train NLL       & Test NLL         \\ \midrule
Normal - Unimodal       & $2.55 \pm 0.01$ & $7.50 \pm 0.88 $ & $2.59 \pm 0.01$ & $4.86 \pm 0.28 $ & $2.63 \pm 0.01$ & $4.30 \pm 0.12 $ & $2.66 \pm 0.01$ & $4.17 \pm 0.08 $ & $2.70 \pm 0.02$ & $4.10 \pm 0.06 $ \\
Heavy-Tailed - Unimodal & $2.87 \pm 0.01$ & $6.31 \pm 0.49 $ & $2.90 \pm 0.01$ & $4.83 \pm 0.17 $ & $2.93 \pm 0.01$ & $4.44 \pm 0.15 $ & $2.97 \pm 0.02$ & $4.23 \pm 0.08 $ & $2.99 \pm 0.02$ & $4.19 \pm 0.06 $ \\
Normal - Bimodal        & $3.16 \pm 0.09$ & $6.81 \pm 1.08 $ & $3.18 \pm 0.08$ & $5.80 \pm 0.37 $ & $3.29 \pm 0.06$ & $5.44 \pm 0.24 $ & $3.31 \pm 0.07$ & $5.32 \pm 0.20 $ & $3.36 \pm 0.05$ & $5.34 \pm 0.09 $ \\ \bottomrule
\end{tabular}}
\caption{Train and test NLL of DGMEs for each toy regression dataset under different dropout probability values.}
\label{tab: dropout_ablation}
\end{table}

\subsubsection{Ablation: Number of Mixture Components}\label{sec:supp-mat-ablation-mixturen}
The number of components in the assumed Gaussian mixture impacts how well the model can estimate more complex noise distributions (e.g., heavy-tailed or bimodal distributions). Gaussian mixtures (with infinite components) are universal approximators to smooth continuous density functions, so the more components assumed, the more flexible the model is. When choosing the number of mixture components, one should take into consideration the complexity of the data generating process and the amount of data in the training set. If the data generating process is known to be Gaussian, then choosing a large number of components is not beneficial. On the other hand, if the data generating process is thought to be multimodal, then using more components is the better choice. We can see this in the following two ablation studies.

Figure \ref{fig: components_ablation_heavy_tailed} shows the effect of the number of mixtures components $K$ on the kurtosis of the learned predictive distribution. We observe that with more mixture components, the the model learns a fatter-tailed distribution. This makes sense since a Student-t distribution can be viewed as a Gaussian mixture with an infinite number of components with different variances. 

\begin{figure}[h]
    \centering
    \includegraphics[width=\linewidth]{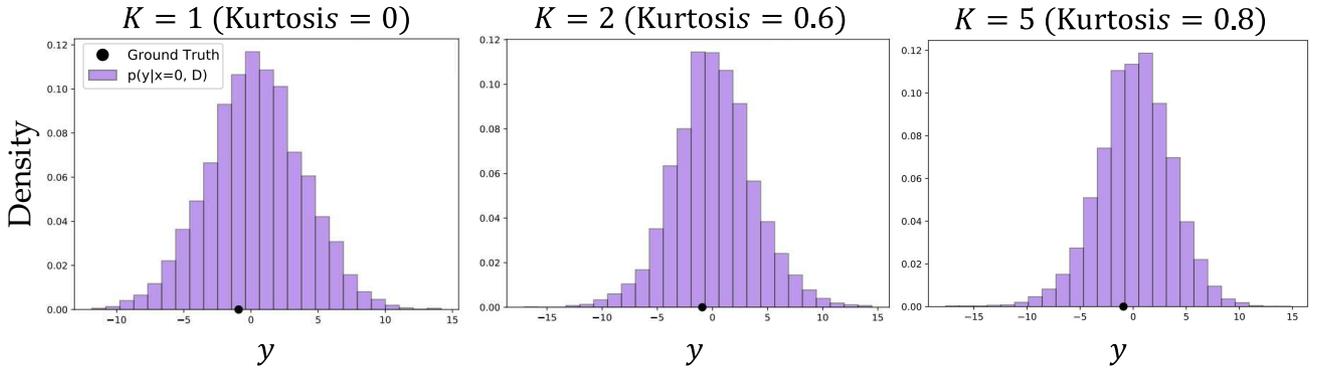}
    \caption{Effect of the number of mixtures on the learned kurtosis of the predictive distribution under heavy-tailed noise. }
    \label{fig: components_ablation_heavy_tailed}
\end{figure}

Figure \ref{fig: components_ablation_bimodal} shows the effect of the number of mixture components on the learned predictive distribution in the case of the bimodal Gaussian. We can see that when DGMEs assumes only $K=1$ mixture component, DGMEs have a similar predictive distribution as DEs, since the model will attempt to explain the bimodal data with a single Gaussian by overestimating the aleatoric noise. An interesting insight is that when DGMEs assume too many components (i.e., $K>2$), the model is still able to accurately learn that the underlying predictive distribution is still bimodal. 

\begin{figure}[t]
    \centering
    \includegraphics[width=0.9\linewidth]{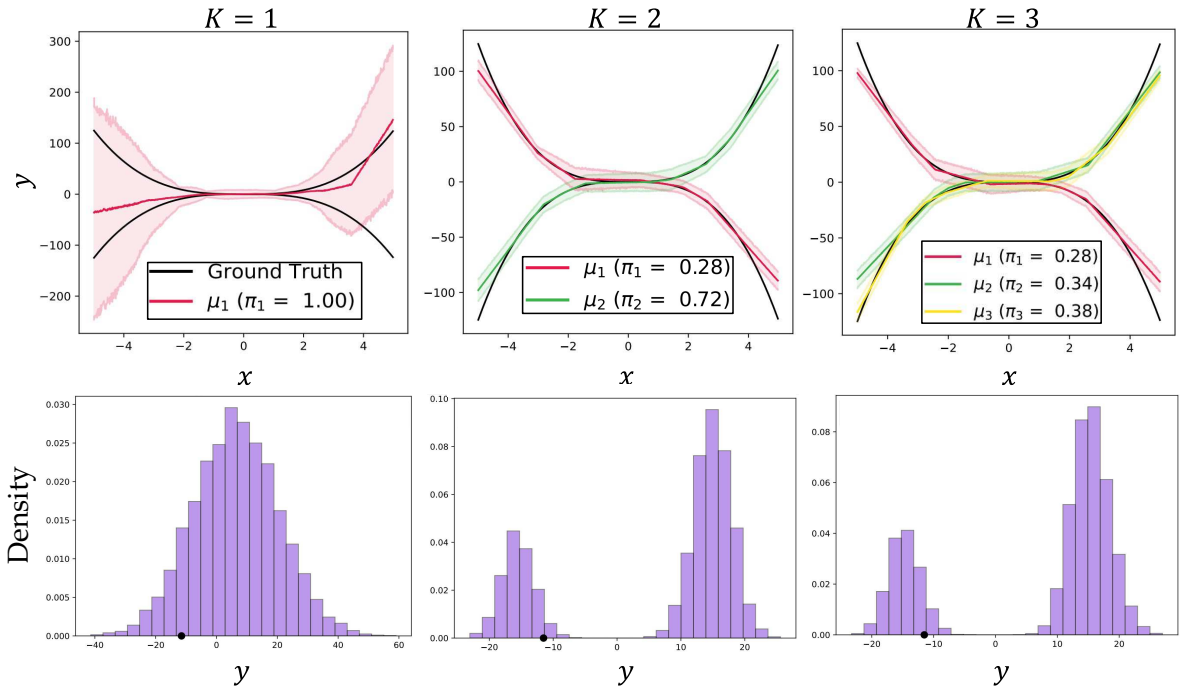}
    \caption{Effect of the number of mixture components on the learned predictive distribution under bimodal noise.}
    \label{fig: components_ablation_bimodal}
\end{figure}

\subsubsection{Ablation: Weight Initialization Schemes and Data Standardization}

To test the impact of weight initialization of the neural network on the performance of DGME, we perform the following ablation study: we train a DGME for 5 rounds, where 10 epochs are used to resolve the M-Step in each round. We use the same architecture as in our toy experiments. We evaluate the NLL on the training set under five different initializations: PyTorch default initialization, initialization with uniform distribution with bounds -0.01 to 0.01, initialization with normal distribution with mean 0 and standard deviation $10^{-6}$, Xavier uniform initialization \cite{glorot2010understanding}, and Xavier normal initialization. As a note, the PyTorch default initialization for a linear layer is done via a uniform distribution ${\cal U}(-\frac{1}{\sqrt{a}}, \frac{1}{\sqrt{a}})$, where $a$ denotes the number of input features to the linear layer. Please refer to \cite{glorot2010understanding} for more information on these weight initialization schemes.

We train the model for each toy dataset over 20 total runs and report the average training NLL and its corresponding standard error. We run this ablation twice: once for training with non-standardized data and once for training with standardized data. The results are shown in Table \ref{tab: init_unstandardized} and Table \ref{tab: init_standardized}. We can see from the results that although weight initialization has some impact on the results, if the data is standardized, it becomes less important. We also see that across all datasets, the default PyTorch initialization gives the most favorable results for both non-standardized and standardized data.

\begin{table}[t]
\centering
\begin{tabular}{@{}cccccc@{}}
\toprule
                        & PyTorch Default & Uniform         & Normal          & Xavier Uniform  & Xavier Uniform  \\ \midrule
Normal - Unimodal       & $2.86 \pm 0.07$ & $3.08 \pm 0.04$ & $3.02 \pm 0.04$ & $2.88 \pm 0.05$ & $2.90 \pm 0.05$ \\
Heavy-Tailed - Unimodal & $3.16 \pm 0.05$ & $3.42 \pm 0.06$ & $3.36 \pm 0.05$ & $3.18 \pm 0.05$ & $3.20 \pm 0.05$ \\
Normal - Bimodal        & $3.36 \pm 0.17$ & $3.56 \pm 0.06$ & $3.61 \pm 0.06$ & $3.46 \pm 0.20$ & $3.47 \pm 0.20$ \\ 
\bottomrule
\end{tabular}
\caption{Impact of different of weight initialization schemes on the train NLL when the data is not standardized. }
\label{tab: init_unstandardized}
\end{table}

\begin{table}[t]
\centering
\begin{tabular}{@{}clllll@{}}
\toprule
                        & \multicolumn{1}{c}{PyTorch Default} & \multicolumn{1}{c}{Uniform} & \multicolumn{1}{c}{Normal} & \multicolumn{1}{c}{Xavier Uniform} & \multicolumn{1}{c}{Xavier Uniform} \\ \midrule
Normal - Unimodal       & $2.55 \pm 0.02$                     & $2.55 \pm 0.01$             & $2.56 \pm 0.01$            & $2.54 \pm 0.01$                    & $2.53 \pm 0.01$                    \\
Heavy-Tailed - Unimodal & $2.87 \pm 0.02$                     & $2.88 \pm 0.01$             & $2.88 \pm 0.01$            & $2.86 \pm 0.01$                    & $2.87 \pm 0.01$                    \\
Normal - Bimodal        & $3.13 \pm 0.07$                     & $3.63 \pm 0.01$             & $3.60 \pm 0.04$            & $3.27 \pm 0.09$                    & $3.24 \pm 0.09$                   \\ 
\bottomrule
\end{tabular}
\caption{Impact of different of weight initialization schemes on the train NLL when the data is standardized. }
\label{tab: init_standardized}
\end{table}

\subsubsection{Illustrative Results: Additive Gaussian Noise}\label{sec: supp-mat-additive-gaussian}

We compare DGMEs with the baselines on the toy regression dataset with Gaussian noise. Figure \ref{fig: sota_comparison_gauss} shows the performance of DGMEs compared to MDNs, MCD and DEs. DGMEs has comparable performance to MCD and DEs and outperforms MDNs.

\begin{figure}[ht]
    \centering
    \includegraphics[width=\linewidth]{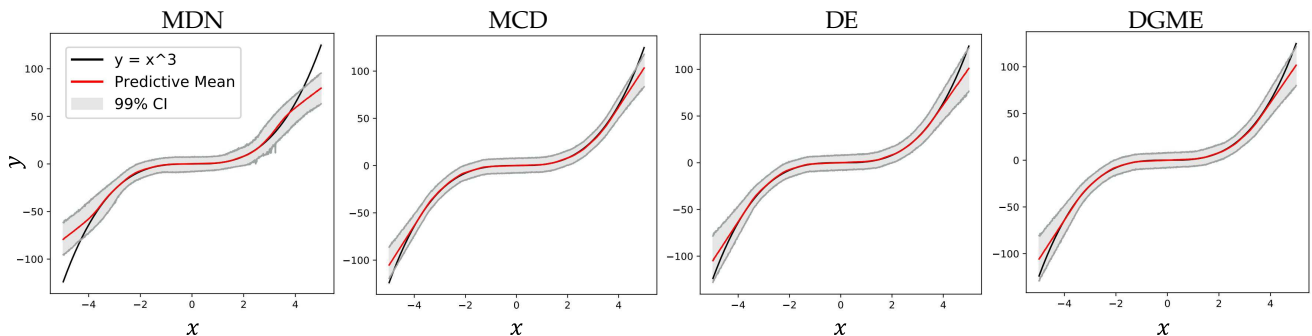}
    \caption{Performance on a toy regression task with Gaussian noise of DGMEs (right) with state-of-the-art methods MDNs, MCD and DEs.}
    \label{fig: sota_comparison_gauss}
\end{figure}

\subsection{Regression on Real Datasets}\label{app: regression}

For real data experiments on a regression task we use the following datasets: (a) Boston Housing dataset\footnote{\url{https://www.kaggle.com/datasets/schirmerchad/bostonhoustingmlnd}}, (b) Concrete compressive strength dataset\footnote{\url{https://archive.ics.uci.edu/ml/datasets/concrete+compressive+strength}} \citep{yeh1998modeling}, (c) Energy efficiency dataset\footnote{\url{https://archive.ics.uci.edu/ml/datasets/energy+efficiency}} \citep{tsanas2012accurate}, (d) Kinematics of an 8 link robot arm dataset \footnote{\url{https://www.openml.org/search?type=data&sort=runs&id=189&status=active}}, (e) 
Combined cycle power plant dataset\footnote{\url{https://archive.ics.uci.edu/ml/datasets/combined+cycle+power+plant}} \citep{tufekci2014prediction}, (f) Wine dataset\footnote{\url{https://archive.ics.uci.edu/ml/datasets/wine}} and (g) Yacht hydrodynamics dataset\footnote{\url{https://archive.ics.uci.edu/ml/datasets/Yacht+Hydrodynamics}}.

In the main text, to provide a fair comparison with techniques that assume the conditional distribution of the data is Gaussian, we summarize the mixture distribution output in both MDNs and DGMEs into a single Gaussian and then evaluate the NLL. This is analagous to the way DEs compute the NLL. We also provide additional results for the test NLL under the assumption of a  mixture of Gaussians in Table \ref{tab:regression_experiments_NLL_mixG} below.

\begin{table*}[t]
    \centering
    \caption{Test NLL for the regression experiments in the mixture of Gaussians case.}
    \label{tab:regression_experiments_NLL_mixG}
    \resizebox{\textwidth}{!}{
    \begin{tabular}{lccccccc}
    \toprule
    \multicolumn{8}{c}{\textbf{\textsc{Test NLL (Mixture of Gaussians)}}}\\
    \midrule
    Dataset                         &            MDNs &           MCD &            DEs &    DGMEs (J=1) &    DGMEs (J=2) &    DGMEs (J=5) &  DGMEs (J=10) \\
    \midrule
    Boston housing                  &    2.71 $\pm$ 0.45 &  2.46 $\pm$ 0.25  &   2.41 $\pm$ 0.25 &  \bf 2.33 $\pm$ 0.18 & \bf  2.33 $\pm$ 0.23 &   2.51 $\pm$ 0.33 &    2.74 $\pm$ 0.53 \\
    Concrete                        &    3.04 $\pm$ 0.22 &   3.04 $\pm$ 0.09 &   3.06 $\pm$ 0.18 &   3.03 $\pm$ 0.10 &   2.99 $\pm$ 0.14 &   2.97 $\pm$ 0.24 & \bf 2.94 $\pm$ 0.22 \\
    Energy                          &  \bf  0.70 $\pm$ 0.17 &   1.99 $\pm$ 0.09 &   1.38 $\pm$ 0.22 &   1.56 $\pm$ 0.14 &   1.31 $\pm$ 0.12 &   0.96 $\pm$ 0.20 &  0.92 $\pm$ 0.48 \\
    Kin8nm                          &  -1.17 $\pm$ 0.04  &  -0.95 $\pm$ 0.03 &  -1.20 $\pm$ 0.02 &  -1.20 $\pm$ 0.02 &  -1.23 $\pm$ 0.03 & \bf -1.24 $\pm$ 0.02 &  \bf  -1.24 $\pm$ 0.02 \\
    Power plant                     &  \bf  2.74 $\pm$ 0.04 &   2.80 $\pm$ 0.05 &   2.79 $\pm$ 0.04 &   2.81 $\pm$ 0.03 &   2.79 $\pm$ 0.03 &   2.77 $\pm$ 0.02 &     2.75 $\pm$ 0.02 \\
    % Protein                         &            NaN &   2.89 $\pm$ 0.01 &   2.83 $\pm$ 0.02 &           NaN &           NaN &           NaN &          NaN \\
    Wine                            &   0.43 $\pm$ 0.86 &   0.93 $\pm$ 0.06 &   0.94 $\pm$ 0.12 &   0.93 $\pm$ 0.12 &   0.90 $\pm$ 0.09 &   0.81 $\pm$ 0.11 &    \bf 0.18 $\pm$ 0.39 \\
    Yacht                           &   0.51 $\pm$ 0.37 &   1.55 $\pm$ 0.12 &   1.18 $\pm$ 0.21 &   0.94 $\pm$ 0.19 &   0.66 $\pm$ 0.18 &  0.51 $\pm$ 0.23 & \bf    0.42 $\pm$ 0.22 \\
    \bottomrule
    \end{tabular}
    }
\end{table*}

\subsection{Hyperparameter Tuning for Financial Forecasting}\label{sec:supp-mat-hyperparam-tuning-finance}
We tuned the hyperparameters of the architecture (number of LSTM layers, number of fully-connected layers, number of LSTM hidden units, number of hidden units in fully-connected layers), optimization procedure (weight decay and learning rate), and the uncertainty quantification associated parameters (dropout probability, and homoscedastic variance value for MCD and MultiSWAG) for each of the approaches using cross validation . We note that all methods use the same feature extractor (LSTM and fully-connected network), which is obtained by hyperparameter tuning each dataset to a single network. To hyperparameter tune, we took the full training period and split it into an ordered sequence of a 90\% training period and a 10\% validation period. We select the hyperparameters based on the combination that maximizes the NLL on the validation period for each dataset.

\section{Possible Extension to Classification Tasks}
\label{app: ood_detection}
Techniques like MDNs and DGMEs are not suited for dealing with classification tasks, since the output of both models is a mixture of Gaussian distributions. For classification tasks, we instead can consider a mixture of categorical distributions, rather than a mixture of Gaussian distributions.  In particular, the conditional distribution $p_{\theta}(y|x)$ is given by
\begin{equation*}
    p_{\theta}(y|x) = \sum_{k=1}^K \pi_k \prod_{i=1}^{d_y} p_{\theta_k}^i(x)^{\mathbb{I}(y=i)},
\end{equation*}
where $p_{\theta_k}^i(x)$ denotes the probability that $y$ belongs to the $i$-th class according to the $k$-th mixture. In this case, we assume MDNs and DGMEs output these probabilities rather than the mean and variance parameterizing a Gaussian distribution.  
\subsection{Entropy Calculation}
To evaluate uncertainty in classification tasks, we consider the average predictive entropy as the metric. To compute the average predictive entropy for a sample $x$, we use the following estimate:
\begin{equation*}
    \widehat{\rm Ent}(x) = -\frac{1}{M} \sum_{m=1}^M \sum_{i\in{\cal C}}\tilde{p}_{(m)}^{i}(x) \log \tilde{p}_{(m)}^{i}(x),
\end{equation*}
where $\tilde{p}_{(m)}^i(x)$ denotes the probability of class $i$ according to the $m$-th sample from the predictive distribution and ${\cal C}$ denotes the set of classes. For both MDNs and DGMEs, these samples are obtained by the following procedure:
\begin{align*}
    k^{(m)} &\sim {\rm Categorical}(\pi_1,\ldots,\pi_K), \\
    \tilde p_{(m)}^i &= p^i_{\theta_{k^{(m)}}}.
\end{align*}
Note that we incorporate dropout in the training procedure of MDNs and DGMEs for this experiment by applying a stochastic forward pass to the sampled network $k^{(m)}$.
\subsection{Example: Uncertainty Evaluation on MNIST}
As an example, we compare DGMEs ability to reason about the underlying uncertainty of new samples with the baseline approaches  with regards to the MNIST handwritten digits dataset. Specifically, for each method, we train a MLP network with 3 hidden layers and 200 hidden units per layer with ReLU activations on the MNIST dataset, including only digits 0-3 and 5-9. After the models are trained, we evaluate the average predictive entropy over three different datasets: the training dataset (known classes), a dataset containing only the digit 4 (unknown classes), and the Fashion-MNIST dataset (unrelated data). We use $M=100$ samples from the predictive distribution to form an estimate of the predictive entropy for each method.We describe in more detail how the average predictive entropy is computed for each method in the Supplementary Material, Section~\ref{sec:supp-mat-hyperparam-tuning-finance}. The results for this experiment are shown in Table \ref{tab: ood_results}, which are averaged over 10 independent runs of each method. The results indicate that DGMEs are able to appropriately reason about the uncertainty in each of the datasets and is competitive with the baseline approaches in each case. DGMEs appropriately obtain that lowest entropy on the training dataset (i.e., the digits it was trained on), obtains a slightly higher entropy on the MNIST dataset containing unknown classes, and the highest entropy on the Fashion-MNIST dataset, which contains examples unrelated to the original classification task.

\begin{table*}[!t]
\centering
\begin{tabular}{lcccc}
\hline
\multicolumn{5}{c}{\textbf{\textsc{Average Predictive Entropy}}}                                                  \\ \hline
Dataset         & MDNs               & MCD               & DEs                & DGMEs              \\ \hline
MNIST (Known)   & $0.019 \pm 0.005$ & $0.012 \pm 0.003$ & $0.012 \pm 0.002$ & $0.015 \pm 0.002$ \\
MNIST (Unknown) & $0.192 \pm 0.032$ & $0.180 \pm 0.020$ & $0.180 \pm 0.020$ & $0.193 \pm 0.016$ \\
Fashion-MNIST   & $0.663 \pm 0.110$ & $0.714 \pm 0.140$ & $0.706 \pm 0.067$ & $0.698 \pm 0.057$ \\
\hline
\end{tabular}
\caption{Average predictive entropy for classification datasets. DGMEs are able to appropriately reason about the underlying uncertainty of OOD samples (MNIST with unknown classes and Fashion-MNIST) and is competitive with respect to state-of-the-art approaches.}
\label{tab: ood_results}
\end{table*}

\section{Sampling from the Predictive Distribution}\label{sec:supp-mat-sampling}
To understand how sampling from the predictive distribution works in DGMEs, we begin with standard formula for determining the predictive distributions in Bayesian models:
\begin{equation*}
    p(y|x, {\cal D}) = \int_{\Theta} p_{\theta}(y|x)p(\theta|{\cal D}) d\theta.
\end{equation*}
In the case of DGMEs, $p_{\theta}(y|x)$ is a mixture of Gaussian distributions and $p(\theta|{\cal D})$ is approximated via dropout.  An important property of the predictive distribution in the case of mixture distributions is that it can be expressed as a mixture of predictive distributions. This property can be derived as follows:
\begin{align*}
    \label{eq: mixture_predictive}
    p(y|x, {\cal D}) &= \int_{\Theta} p_{\theta}(y|x)p(\theta|{\cal D})d\theta \\
    &= \int_{\Theta} \left(\sum_{k=1}^K \pi_k p_k(y|x, \theta_k)\right) p(\theta|{\cal D}) d\theta \\
    &= \int_{\Theta} \sum_{k=1}^K \pi_k p_k(y|x, \theta_k)p(\theta|{\cal D}) d\theta \\
    &= \sum_{k=1}^K \pi_k \int_{\Theta} p_k(y|x, \theta_k)p(\theta|{\cal D}) d\theta \\
    &= \sum_{k=1}^K \pi_k \int_{\Theta_k} p_k(y|x, \theta_k)d\theta_k  \underbrace{\int_{\Theta_{-k}} p(\theta|{\cal D}) d\theta_{-k}}_{{p(\theta_k|{\cal D})}} \\
    &= \sum_{k=1}^K \pi_k \int_{\Theta_k} p_k(y|x, \theta_k) p(\theta_k|{\cal D}) d\theta_k.
\end{align*}
Since $p_k(y|x, {\cal D}) = \int_{\Theta_k} p_k(y|x, \theta_k)p(\theta_k|{\cal D})d\theta_k$, we obtain the following expression for the predictive distribution:
\begin{equation*}
    p(y|x, {\cal D}) = \sum_{k=1}^K \pi_k p_k(y|x, \theta_k).
\end{equation*}
This form implies that we can draw samples from the predictive distribution in DGMEs using the following procedure:
\begin{enumerate}
    \item Sample the mixture component: 
    $$k\sim{\rm Categorical}(\pi_1,\ldots,\pi_K)$$
    \item Sample the posterior parameters of the given mixture component. In this work, dropout was used to approximate each posterior $p(\theta_k|{\cal D})$:
    \begin{align*}
    a_{k, i} &\sim {\rm Bernoulli}(p_d), \quad i=1,\ldots, d_{\theta}, \\
    \theta_k &= a_k \odot \theta_k^\star
    \end{align*}
    \item Draw the sample of $y$ from the appropriate predictive distribution:
    $$ y \sim  p_k(y|x, \theta_k)$$
\end{enumerate}

% \subsubsection*{Disclaimer}

% \paragraph{Disclaimer}
% This paper was prepared for informational purposes by the Artificial Intelligence Research group of JPMorgan Chase \& Co. and its affiliates (``JP Morgan''), and is not a product of the Research Department of JP Morgan. JP Morgan makes no representation and warranty whatsoever and disclaims all liability, for the completeness, accuracy or reliability of the information contained herein. This document is not intended as investment research or investment advice, or a recommendation, offer or solicitation for the purchase or sale of any security, financial instrument, financial product or service, or to be used in any way for evaluating the merits of participating in any transaction, and shall not constitute a solicitation under any jurisdiction or to any person, if such solicitation under such jurisdiction or to such person would be unlawful.

\section{Comparison of Uncertainty Quantification Approaches}\label{sec:supp-mat-comp}
Here, we provide an overall comparison of the benchmarks used in the experiments of this work as compared to the proposed approach along different qualities: the likelihood assumption, whether or not mixture weights are learned, how aleatoric uncertainty is quantified, and how epistemic uncertainty is quantified. This comparison is provided in Table \ref{tab: compare_methods}

\begin{table}[b]
\centering
\resizebox{\textwidth}{!}{\begin{tabular}{@{}lp{1in}p{1in}p{1.3in}p{1.35in}p{1.35in}@{}}
\toprule
\textbf{Method} &
  \textbf{Likelihood} &
  \textbf{Mixture Weights} &
  \textbf{Aleatoric Uncertainty} &
  \textbf{Epistemic Uncertainty} &
  \textbf{Other Notes} \\ \midrule
MDNs &
  Mixture of Gaussians &
  Learned and input dependent &
  Heteroscedastic &
  None in original implementation, but dropout is applied for fair comparison in this implementation &
  Off-the-shelf can be applied to account for epistemic uncertainty (e.g., dropout, Laplace approximation, SWAG, variational Bayes, etc.) \\
  \midrule
MCD &
  Gaussian &
  Each prediction made via a stochastic forward pass at test time is equally weighted. &
  Homoscedastic &
  Dropout &
   \\
   \midrule
DEs &
  Gaussian &
  Assumed uniform &
  Heteroscedastic &
 Adversarial training and weight initialization. Dropout is also applied in this implementation using hyperparameter optimization. &
   \\
   \midrule
MultiSWAG &
  Gaussian &
  Assumed uniform &
  Homoscedastic &
  Stochastic weight averaging Gaussian (SWAG) &
  One can also account for heteroscedasticity by applying SWAG training to a deep ensemble that outputs a mean and variance \\
  \midrule
DGMEs &
  Mixture of Gaussians &
  Learned and independent of input &
  Heteroscedastic &
  Dropout in this implementation &
  Other methods to account for epistemic can be used off-the-shelf (e.g., Laplace approximation, SWAG, variational Bayes, etc.) \\ \bottomrule
\end{tabular}}
\caption{Summary of benchmarks as compared to DGMEs.}
\label{tab: compare_methods}
\end{table}

\end{document}